\documentclass{article} 
\usepackage{times}  
\usepackage{helvet}  
\usepackage{courier}  
\usepackage[hyphens]{url}  
\usepackage{graphicx} 
\usepackage{natbib}  
\usepackage{caption} 
\usepackage[margin = 1in]{geometry}
\usepackage{authblk}
\usepackage{algorithm}

\usepackage{newfloat}
\usepackage{listings}
\DeclareCaptionStyle{ruled}{labelfont=normalfont,labelsep=colon,strut=off} 
\floatstyle{ruled}
\newfloat{listing}{tb}{lst}{}
\floatname{listing}{Listing}

\usepackage{algpseudocode}
\usepackage{csquotes}
\usepackage{subcaption}
\usepackage{tikz}
\usetikzlibrary{positioning,arrows,fit}
\usetikzlibrary{backgrounds,automata,calc}
\usetikzlibrary{decorations.pathreplacing}
\usepackage{hyperref}

\usepackage{paralist}

\usepackage{amsmath,amsfonts,amsthm}

\usepackage{booktabs}
\usepackage{multirow}

\DeclareMathOperator*{\argmin}{arg\,min}

\usepackage{enumerate}
\usepackage{enumitem}

\usepackage{tabularx}
\usepackage{appendix}
\usepackage{cleveref}
\usepackage{thmtools}
\usepackage{mathtools}
\usepackage{dsfont}
\usepackage{booktabs}

\newcommand{\R}{\mathbb{R}}

\renewcommand{\cal}[1]{\mathcal{#1}}

\newtheorem{theorem}{Theorem}[section]

\theoremstyle{definition}

\theoremstyle{definition}

\newtheorem{example}[theorem]{Example}

\title{Axiomatic Aggregations of Abductive Explanations}
\author[1]{Gagan Biradar}
\author[2]{Yacine Izza}
\author[1]{Elita Lobo}
\author[1]{Vignesh Viswanathan}
\author[1]{Yair Zick}
\affil[1]{University of Massachusetts, Amherst, USA} 
\affil[2]{CREATE, NUS, Singapore}
\date{}

\begin{document}

\maketitle

\begin{abstract}
The recent criticisms of the robustness of post hoc model approximation explanation methods (like LIME and SHAP) have led to the rise of model-precise abductive explanations. For each data point, abductive explanations provide a minimal subset of features that are sufficient to generate the outcome. 
While theoretically sound and rigorous, abductive explanations suffer from a major issue --- there can be several valid abductive explanations for the same data point. In such cases, providing a single abductive explanation can be insufficient; on the other hand, providing all valid abductive explanations can be incomprehensible due to their size. In this work, we solve this issue by aggregating the many possible abductive explanations into feature importance scores. We propose three aggregation methods: two based on power indices from cooperative game theory and a third based on a well-known measure of causal strength. We characterize these three methods axiomatically, showing that each of them uniquely satisfies a set of desirable properties. We also evaluate them on multiple datasets and show that these explanations are robust to the attacks that fool SHAP and LIME.
\end{abstract}

\section{Introduction}\label{sec:intro}
The increasing use of complex machine learning (predictive) models in high-stake domains like finance \citep{ozbayoglu2020finance} and healthcare \citep{pandey2022healthcare, qayyum2021healthcare} necessitates the design of methods to accurately explain the decisions of these models. 
Many such methods have been proposed by the AI community.
Most of these methods (like SHAP \citep{Lundberg2017Unified} and LIME \citep{Ribeiro2016Lime}) explain model decisions by sampling points and evaluating model behavior around a point of interest. 
While useful in many settings, this class of model approximation-based methods has faced criticisms for being unable to fully capture model behavior \citep{rudin2019stop,huang2023inadequacy} and being easily manipulable \citep{Slack2020Shap}. 
The main issue with these methods stems from the fact that model approximation-based explanation measures use the model's output on a small fraction of the possible input points.
This has led to the rise of model-precise {\em abductive explanations} \citep{shih2018abductive, ignatiev2019abductiveog} which use the underlying model's structure to compute rigorous explanations. 
Abductive explanations are simple: they provide a minimal set of features that are sufficient to generate the outcome. 
In other words, a set of features $S$ forms an abductive explanation for a particular point of interest $\vec x$ if no matter how we modify the values of the features outside $S$, the outcome will not change.

Despite being simple, concise, and theoretically sound, abductive explanations suffer from a major flaw --- there may be several possible abductive explanations for a given data point. 
Consider the following example:

\begin{example}
Suppose that we train a simple rule-based model $f$ for algorithmic loan approval, using the features `Age', `Purpose', `Credit Score', and `Bank Balance'. 
The rule-based model has the following closed-form expression:
\begin{align*}
    f(\vec x) = & (\text{Age} > 20 \land \text{Purpose} = \text{Education}) \\ &\lor (\text{Credit} > 700) \lor (\text{Bank} > 50000)
\end{align*}
In simple words, if the applicant has an age greater than 20 and is applying for education purposes, the loan is accepted; otherwise, if the applicant has a credit score greater than 700 or a bank account balance greater than 50000, the loan is accepted. 

Consider a user with the following details $\vec x = (\text{Age} = 30, \text{Purpose} = \text{Education}, \text{Credit} = 750, \text{Bank} = 60000)$. There are three abductive explanations for this point: (Age, Purpose), (Credit), and (Bank).
\end{example}

In this example, if we provide the abductive explanation (Age, Purpose) to the user, they can infer that their age and purpose played a big role in their decision. However, note that it would be incorrect to infer anything else. The user cannot even tell if the features which are absent from the explanation played any role in their acceptance. In fact, the user still does not know whether the feature Age (present in the explanation) was more important than the feature Credit Score (absent in the explanation). Arguably, Credit Score is more relevant than Age since it is present in a singleton abductive explanation. However, no user presented with only one abductive explanation can make this conclusion. 

We propose to aggregate abductive explanations into importance scores for each feature. 
Feature importance scores are an extremely well-studied class of explanations \citep{barocas2020explanations}. 
As seen with the widespread use of measures like SHAP and LIME, the simple structure of feature importance scores make it easy to understand and visualize. 
We propose to use these feature importance scores to give users a comprehensive understanding of model behavior that is impossible to obtain from a single abductive explanation.

\subsection{Our Contributions}\label{sec:contrib}
\textbf{Conceptual:} 
We present three aggregation measures --- the Responsibility Index, the Deegan-Packel Index, and the Holler-Packel Index (Section \ref{sec:framework}). The Responsibility index is based on the degree of responsibility --- a well-known causal strength quantification metric \citep{chockler2004responsibility}. The Deegan-Packel and Holler-Packel indices are based on power indices from the cooperative game theory literature \citep{DeeganPackel1978, Holler1982OG, HollerPackel1983}.

\noindent \textbf{Theoretical:} For each of these measures, we present an axiomatic characterization, in line with theoretical results in the model explainability community \citep{Patel2021Bii, Lundberg2017Unified, Datta2016Qii, Sundarajan2020Shapley}. Since we deal with aggregating abductive explanations as opposed to conventional model outputs, our proof styles and axioms are novel.

\noindent \textbf{Empirical:} We empirically evaluate our measures, comparing them with well-known feature importance measures: SHAP \citep{Lundberg2017Unified} and LIME \citep{Ribeiro2016Lime}. 
Our experimental results (Section \ref{sec:expts}) demonstrate the robustness of our methods, showing specifically that they are capable of identifying biases in a model that SHAP and LIME cannot identify. 

\subsection{Related Work}\label{sec:related}
Abductive explanations were first formally defined in \citet{ignatiev2019abductiveog} as a generalization of prime implicant explanations defined in \citet{shih2018abductive}.
For most commonly used machine learning models models, computing abductive explanations is an intractable problem; hence, computing abductive explanations for these models often requires using NP oracles (e.g. SAT/SMT, MILP, etc).

These oracles have been used in different ways to compute abductive explanations for different classes of models. For example, MILP-encodings have been used for neural networks \citep{ignatiev2019abductiveog} and SMT-encodings have been used for tree ensembles \citep{silva2022treeensembles}. For less complex models such as monotonic classifiers and naive bayes classifiers, polynomial time algorithms to compute abductive explanations are known \citep{silva2020naivebayes, silva2021monotonic}.

The main focus of these papers has been the runtime of the proposed algorithms. There are fewer papers analysing the quality of the output abductive explanations. Notably, the work of \citet{audemard2022preferred} is also motivated by the fact that there can be several abductive explanations for a single data point; however, the solution they propose is radically different from ours. They propose using the explainer's preferences over the set of explanations to find a {\em preferred} abductive explanation to provide to the user. 

More recently, \citet{huang2023inadequacy} observe that SHAP \citep{Lundberg2017Unified} often fails to identify features that are irrelevant to the prediction of a data point, i.e.\ assigns a positive score to features that never appear in any abductive explanations. They propose aggregating abductive explanations as an alternative to SHAP but do not propose any concrete measures to do so.
Our work answers this call with three axiomatically justified aggregation measures.

Parallel to our work\footnote{The work of \citet{Ignatiev-corr23} was developed independently and at the same time as ours, but we preferred to wait before we made our work public on arXiv.}, the work of \citet{Ignatiev-corr23} also builds on the observations of \citet{huang2023inadequacy} and develops a MARCO-like method~\citet{LiffitonPMM16} for computing feature importance explanations by aggregating abductive explanations.
Their work proposes two aggregation measures, {\it formal feature attribution (ffa)} and {\it weighted ffa}, that correspond exactly to the Holler-Packel and Deegan-Packel indices respectively. We remark, however, that their work does not offer an axiomatic characterization of these measures, and focuses primarily on empirical performance.

There has also been recent work generalizing abductive explanations to {\em probabilistic abductive explanations} \citep{waldchen2021probabilistic, ArenasBOS22, izza-ijar23}. Probabilistic abductive explanations allow users to trade-off precision for size, resulting in smaller explanations with lower precision i.e. smaller explanations which are not as robust as abductive explanations.

Our work also contributes novel feature importance measures. Feature importance measures have been well studied in the literature with measures like SHAP \citep{Lundberg2017Unified} and LIME \citep{Ribeiro2016Lime} gaining significant popularity. 
There are several other measures in the literature, many offering variants of the Shapley value \citep{Sundarajan2020Shapley, Frye2020Assymetric, Sundarajan2017NeuralNetworks}. Other works use the Banzhaf index \citep{Patel2021Bii} and necessity and sufficiency scores \citep{Galhotra2021Contrastive, watson2021local}.


\section{Preliminaries}\label{sec:preliminaries}
We denote vectors by $\vec x$ and $\vec y$. We denote the $i$-th and $j$-th indices of the vector $\vec x$ using $x_i$ and $x_j$. 
Given a set $S$, we denote the restricted vector containing only the indices $i \in S$ using $\vec x_S$. 
We also use $[k]$ to denote the set $\{1, 2, \dots, k\}$.

We have a set of features $N = \{1, 2, \dots, n\}$, where each $i \in N$ has a domain $\cal X_i$. We use $\cal X = \bigtimes_{i \in N} \cal X_i$ to denote the domain of the input space.
We are given a {\em model of interest} $f \in \cal F$ that maps input vectors $\vec x \in \cal X$ to a binary output variable $y \in \{0, 1\}$. 
In the local post hoc explanation problem, we would like to explain the output of the model of interest $f$ on a {\em point of interest} $\vec x$.
We work with two forms of model explanations in this paper.

The first is that of {\em feature importance weights} (or {\em feature importance scores}): feature importance weights provide a score to each feature proportional to their importance in the generation of the outcome $f(\vec x)$. Commonly used feature importance measures are LIME \citep{Ribeiro2016Lime} and SHAP \citep{Lundberg2017Unified}.

Second, an {\em abductive explanation} for a point of interest $\vec x$ is a minimal subset of features which are sufficient to generate the outcome $f(\vec x)$. More formally, an abductive explanation (as defined by \citet{silva2022treeensembles}) corresponds to a subset minimal set of features $S$ such that:
\begin{align}
    \forall \vec y \in \cal X, \,\, \big ( \vec y_S = \vec x_S \big ) \implies \big ( f(\vec y) = f(\vec x) \big ) \label{eq:abductive-explanation}
\end{align}

By subset minimality, if $S$ satisfies \eqref{eq:abductive-explanation}, then no proper subset of $S$ satisfies \eqref{eq:abductive-explanation}. We use $\cal M(\vec x, f)$ to denote the set of abductive explanations for a point of interest $\vec x$ under a model of interest $f$. We also use $\cal M_i(\vec x, f)$ to denote the subset of $\cal M(\vec x, f)$ containing all the abductive explanations with the feature $i$.
Our goal is to create aggregation measures that maps $\cal M(\vec x, f)$ to an importance score for each feature $i \in N$.

\subsection{A Cooperative Game Theory Perspective}\label{subsec:set-function}
In this paper, we propose to aggregate abductive explanations into feature importance scores. A common approach used to compute feature importance scores is via modeling the problem as a cooperative game \citep{Patel2021Bii, Datta2016Qii, Lundberg2017Unified}.
This formulation allows us to both, tap into the existing literature on power indices (like the Shapley value) to create feature importance measures, as well as use theoretical techniques from the literature to provide axiomatic characterizations for new measures. In this paper, we do both.

A simple cooperative game \citep{Wooldridge2011} $(N, v)$ is defined over a set of players $N$ and a monotone\footnote{Recall that a set function $v$ is monotonic if for all $S \subseteq T \subseteq N$, $v(S) \le v(T)$.} binary value function $v:2^N \mapsto \{0, 1\}$. The set of players, in our setting (and several others \citep{Patel2021Bii, Datta2016Qii, Lundberg2017Unified}), are the features of the model of interest $N$. 
The value function $v$ loosely represents the value of each (sub)set of players; in model explanations, the value function represents the joint importance of a set of features in generating the outcome. 

A set $S \subseteq N$ is referred to as a {\em minimal winning set} if $v(S) = 1$ and for all proper subsets $T \subset S$, $v(T) = 0$. 
Minimal winning sets are a natural analog of abductive explanations in the realm of cooperative game theory. There are specific power indices like the Deegan-Packel index \citep{DeeganPackel1978} and the Holler-Packel index \citep{HollerPackel1983, Holler1982OG} which take as input the set of all minimum winning sets and output a score corresponding to each player (in our case, feature) in the cooperative game. These measures are natural candidates to convert abductive explanations into feature importance scores.

\section{A Framework for Abductive Explanation Aggregation}\label{sec:framework}

\begin{table*}[t]
    \centering
    \begin{tabularx}{\textwidth}{>{\hsize = 0.33\hsize}X>{\hsize = 0.33\hsize}X>{\hsize = 0.33\hsize}X}
        \toprule
         \textbf{Measure} &  \textbf{$\alpha$-Monotonicity} & \textbf{$C$-Efficiency}  \\
        \midrule
        Holler-Packel Index \newline $\eta_i(\vec x, f) = |\cal M_i(\vec x, f)|$ & $\alpha(\cal S) = \cal S$ and $\alpha(\cal S) \le \alpha(\cal T)$ iff $\cal S \subseteq \cal T$ & $C(\vec x, f) = \sum_{i \in N} |\cal M_i(\vec x, f)|$ \\
        \midrule
         Deegan-Packel Index \newline $\phi_i(\vec x, f) = \sum_{S \in \cal M_i(\vec x, f)}\frac1{|S|}$ & $\alpha(\cal S) = \cal S$ and $\alpha(\cal S) \le \alpha(\cal T)$ iff $\cal S \subseteq \cal T$ & $C(\vec x, f) = |\cal M(\vec x, f)|$  \\
         \midrule
         Responsibility Index \newline $\rho_i(\vec x, f) = \max_{S \in \cal M_i(\vec x, f)}\frac1{|S|}$ & $\alpha(\cal S) = -\min_{S \in \cal S} |S|$ & NA  \\
        \bottomrule
    \end{tabularx}
    \vspace{0.2cm}
    \caption{A summary of the $\alpha$ and $C$ values from the Monotonicity and Efficiency properties respectively of each measure defined in this paper. All three measures satisfy Symmetry and Null Feature. The Responsibility index satisfies an alternative efficiency property which is incomparable to $C$-efficiency.}
    \label{tab:property-summary}
\end{table*}

Formally, we define an {\em abductive explanation aggregator} (or simply an {\em aggregator}) as a function that maps a point $\vec x$ and a model $f$ to a vector in $\R^n$ using only the abductive explanations of the point $\vec x$ under the model $f$; the output vector can be interpreted as importance scores for each feature. 
For any arbitrary aggregator $\beta: \cal{X} \times \cal{F} \rightarrow \R^n$, we use $\beta_i(\vec x, f)$ as the importance weight given to the $i$-th feature for a specific datapoint-model pair $(\vec x, f)$. 

In order to design meaningful aggregators, we take an axiomatic approach: we start with a set of desirable properties and then find the unique aggregator which satisfies these properties. This is a common approach in explainable machine learning \citep{Datta2016Qii, Lundberg2017Unified, Sundarajan2017NeuralNetworks, Patel2021Bii}. The popular Shapley value \citep{Young1985Shapley} is the unique measure that satisfies four desirable properties --- Monotonicity, Symmetry, Null Feature, and Efficiency.

However, the exact definitions of these four properties in the characterization of the Shapley value do not extend to our setting (see Appendix \ref{apdx:shapley-value}). Moreover, the Shapley value does not aggregate abductive explanations (or more generally, minimal winning sets). Therefore, for our axiomatic characterization, we formally define variants of these properties, keeping the spirit of these definitions intact. We present these definitions below.

\noindent \textbf{$\alpha$-Monotonicity:} Let $\alpha$ be some function that quantifies the relevance of a set of abductive explanations a feature $i$ is present in. A feature importance score is monotonic with respect to $\alpha$ if for each feature $i$ and dataset model pair $(\vec x, f)$, the importance score given to $i$ is monotonic with respect to $\alpha(\cal M_i(\vec x, f))$. 

In simple words, the higher the rank of the set of abductive explanations containing a feature (according to $\alpha$), the higher their importance scores. 
The ranking function $\alpha$ can capture several intuitive desirable properties. 
For example, if we want features present in a larger number of abductive explanations to receive a higher score, we can simply set $\alpha(\cal S) = |\cal S|$. Otherwise, if we want features present in smaller explanations to receive a higher score, we set $\alpha(\cal S) = - \min_{S \in \cal S} |S|$.

Formally, let $\alpha: 2^{2^N} \mapsto \cal Y$ be a function that ranks sets of abductive explanations, i.e., maps every set of abductive explanations to a partially ordered set $\cal Y$. 
An aggregator $\beta$ is said to satisfy $\alpha$-monotonicity if for any two datapoint-model pairs $(\vec x, f)$ and $(\vec y, g)$ and a feature $i$, $\alpha(\cal M_i(\vec x, f)) \le  \alpha(\cal M_i(\vec y, g))$ implies $\beta_i(\vec x, f) \le \beta_i(\vec y, g)$. 
Additionally,  if the feature $i$ has the same set of abductive explanations under $(\vec x, f)$ and $(\vec y,g)$ --- i.e., $\cal M_i(\vec x, f) = \cal M_i(\vec y, g)$ --- then $\beta_i(\vec x, f) = \beta_i(\vec y, g)$.

\noindent \textbf{Symmetry:} This property requires that the index of a feature should not affect its score. That is, the score of feature $i$ should not change if we change its position.  
Given a permutation $\pi:N \rightarrow N$, we define $\pi \vec x$ as the reordering of the feature values in $\vec x$ according to $\pi$. 
In addition, given a permutation $\pi:N \rightarrow N$, we define $\pi f$ as the function that results from permuting the input point using $\pi$ before computing the output. More formally, $\pi f(\vec x) = f(\pi \vec x)$. We are now ready to formally define the symmetry property:

An aggregator $\beta$ satisfies symmetry if for any datapoint-model pair $(\vec x, f)$ and a permutation $\pi$, $\pi \beta(\vec x, f) = \beta(\pi \vec x, \pi^{-1} f)$.

\noindent \textbf{Null Feature:} if a feature is not present in {\em any} abductive explanation, it is given a score of $0$. 
This property explicitly sets a baseline value for importance scores. 
More formally, an aggregator $\eta$ satisfies Null Feature if for any datapoint-model pair $(\vec x, f)$ and any feature $i$, $\cal M_i(\vec x, f) = \emptyset$ implies that $\eta_i(\vec x, f) = 0$.

\noindent \textbf{$C$-Efficiency:} This property requires the scores output by aggregators to sum up to a {\em fixed value}; in other words, for any datapoint-model pair $(\vec x, f)$, $\sum_{i \in N} \beta_i(\vec x, f)$ must be a fixed value. 
Not only does efficiency bound the importance scores, but it also ensures that features are not always given a trivial score of $0$. 
The fixed value may depend on the aggregator $\beta$, the model $f$, and the datapoint $\vec x$. 
To capture this, we define a function $C$ that maps each datapoint-model pair $(\vec x, f)$ to a real value.

An aggregator $\beta$ is $C$-efficient if for any datapoint-model pair $(\vec x, f)$, $\sum_{i \in N} \beta_i(\vec x, f) = C(\vec x, f)$.

We deliberately define the above properties flexibly. 
There are different reasonable choices of $\alpha$-monotonicity and $C$-efficiency --- each leading to a different aggregation measure (Table~\ref{tab:property-summary}). 
In what follows, we formally present these choices and mathematically find the measures they characterize.
It is worth noting, as shown by \citet{huang2023inadequacy}, that the popular SHAP framework fails to satisfy the Null Feature property while all the measures we propose in this paper are guaranteed to satisfy the Null Feature property. 

\subsection{The Holler-Packel Index}\label{subsec:holler-packel}
We start with the Holler-Packel index, named after the power index in cooperative game theory \citep{Holler1982OG, HollerPackel1983}. The Holler-Packel index measures the importance of each feature as the number of abductive explanations that contain it. More formally, the Holler-Packel index of a feature $i$ (denoted by $\eta_i(\vec x, f)$) is given by
\begin{align}
    \eta_i(\vec x, f) = |\cal M_i(\vec x, f)| \label{eq:holler-packel}
\end{align}

The Holler-Packel index satisfies a property we call {\em Minimal Monotonicity}. This property corresponds to $\alpha$-Monotonicity when $\alpha(\cal S) = \cal S$ and $\alpha(S) \le \alpha(T)$ if and only if $\cal S \subseteq \cal T$. Minimal Monotonicity (loosely speaking) ensures that features present in a larger number of abductive explanations get a higher importance score. 

The Holler-Packel index also satisfies $C$-Efficiency where $C(\vec x, f)$ is defined as $\sum_{i \in N} |\cal M_i(x, f)|$. We refer to this property as $(\sum_{i \in N} |\cal M_i(x, f)|)$-Efficiency for clarity. 

Our first result shows that the Holler-Packel index is the only index that satisfies Minimal Monotonicity, Symmetry, Null Feature, and $(\sum_{i \in N} |\cal M_i(x, f)|)$-Efficiency.

\begin{restatable}{theorem}{thmhollerpackelaxioms}\label{thm:holler-packel-axioms}
The only aggregator that satisfies Minimal Monotonicity, Symmetry, Null Feature, and $(\sum_{i \in N} |\cal M_i(x, f)|)$-Efficiency is the Holler-Packel index given by \eqref{eq:holler-packel}.
\end{restatable}

The Holler-Packel index was used as a heuristic abductive explanation aggregator in prior work under the term `hit rate' \citep{silva2020naivebayes}. Theorem \ref{thm:holler-packel-axioms} theoretically justifies the hit rate.

\subsection{The Deegan-Packel Index}\label{subsec:deegan-packel}
Next, we present the Deegan-Packel index. 
This method is also named after the similar game-theoretic power index \citep{DeeganPackel1978}. 
The Deegan-Packel index, like the Holler-Packel index, counts the number of abductive explanations a feature is included in but unlike the Holler-Packel index, each abductive explanation is given a weight inversely proportional to its size. This ensures that smaller abductive explanations are prioritized over larger abductive explanations. 
Formally, the Deegan-Packel index is defined as follows:
\begin{align}
    \phi_i(\vec x, f) = \sum_{S \in \cal M_i(\vec x, f)}\frac{1}{|S|} \label{eq:deegan-packel}
\end{align}

Note that this aggregator also satisfies Minimal Monotonicity, Symmetry, and Null Feature. However, the Deegan-Packel index satisfies a different notion of $C$-Efficiency. The efficiency notion satisfied by the Deegan-Packel index corresponds to $C$-Efficiency where $C(\vec x, f)$ is defined as $|\cal M(\vec x, f)|$. We refer to this efficiency notion as \newline $|\cal M(\vec x, f)|$-Efficiency for clarity.

Our second result shows that the Deegan-Packel index uniquely satisfies Minimal Monotonicity, Symmetry, Null Feature, and $|\cal M(\vec x, f)|$-Efficiency.
\begin{restatable}{theorem}{thmdeeganpackelaxioms}\label{thm:deegan-packel-axioms}
The only aggregator that satisfies Minimal Monotonicity, Symmetry, Null Feature, and $|\cal M(\vec x, f)|$-Efficiency is the Deegan-Packel index given by \eqref{eq:deegan-packel}.
\end{restatable}

\subsection{The Responsibility Index}\label{subsec:responsibility}
We now present our third and final aggregator, the Responsibility index, named after the degree of responsibility \citep{chockler2004responsibility} used to measure causal strength.

The Responsibility index (denoted by $\rho$) of a feature is the inverse of the size of the smallest abductive explanation containing that feature. 
More formally,   
\begin{align}
    \rho_i(\vec x, f) = 
    \begin{cases}
        \max_{S \in \cal M_i(\vec x, f)} \frac1{|S|} & \cal M_i(\vec x, f) \ne \emptyset \\
        0 & \cal M_i(\vec x, f) = \emptyset
    \end{cases}
 \label{eq:responsibility}
\end{align}

To characterize this aggregator, we require different versions of Monotonicity and Efficiency. Our new monotonicity property requires aggregators to provide a higher score to features present in smaller abductive explanations. 
We refer to this property as Minimum Size Monotonicity: this corresponds to $\alpha$-Monotonicity where given a set of abductive explanations $\cal S$, we let $\alpha(\cal S) = - \min_{S \in \cal S} |S|$.

The new efficiency property does not fit into the $C$-Efficiency framework used so far and is easier to define as two new properties --- Unit Efficiency and Contraction. Unit Efficiency requires that the score given to any feature present in a singleton abductive explanation be $1$. This property is used to upper bound the score given to a feature.

\noindent \textbf{Unit Efficiency: } For any datapoint-model pair $(\vec x, f)$, $\cal M_i(\vec x, f) = \{\{i\}\}$ implies $\rho_i(\vec x, f) = 1$.

To define the contraction property, we define the \emph{contraction operation} on the set of features $N$: we replace a subset of features $T\subseteq N$ by a single feature $[T]$ corresponding to the set. 
The \emph{contracted data point} $\vec x^{[T]}$ is the same point as $\vec x$, but we treat all the features in $T$ as a single feature $[T]$. 
The contraction property requires that a contracted feature $[T]$ does not receive a score greater than the sum of the scores given to the individual features in $T$.

\textbf{Contraction:} For any subset $T$ that does not contain a null feature (i.e., a feature not included in any abductive explanation), we have $\rho_{[T]}(\vec x^{[T]}, f) \le \sum_{i \in T} \rho_i(\vec x, f)$. 
Moreover, equality holds if $T \in \{S : S \in \argmin_{S' \in \cal M_i(\vec x, f)}|S'|\}$ for all $i \in T$. In other words, equality holds iff $T$ is the smallest abductive explanation for all the features in $T$. 

The contraction property bounds the gain one gets by combining features and ensures that the total attribution that a set of features receives when combined does not exceed the sum of the individual attributions of each element in the set.

We are now ready to present our characterization of the Responsibility index.

\begin{restatable}{theorem}{thmresponsibilityindexaxioms}\label{thm: responsibility-index-axioms}
The Responsibility index is the only aggregator which satisfies Minimum Size Monotonicity, Unit Efficiency, Contraction, Symmetry, and Null Feature.
\end{restatable}

 \subsection{Impossibilities}
The framework discussed above can be used to axiomatically characterize several indices. 
Our axiomatic approach also offers insights as to what \emph{can} be accomplished by aggregating abductive explanations. We prove that some choices of $\alpha$ and $C$ may create a set of properties that are impossible to satisfy simultaneously. For example, the Shapley value's efficiency property stipulates that all Shapley values must sum to $1$. Somewhat surprisingly, this is not possible when taking an abductive explanation approach.

\begin{restatable}{prop}{propdeeganpackelimpossibility}\label{prop:deegan-packel-impossibility}
There exists no aggregator satisfying Minimal Monotonicity, Symmetry, Null Feature, and 1-Efficiency.
\end{restatable}

All the indices described in this section inherit the precision and robustness of abductive explanations while simultaneously satisfying a set of desirable properties. In what follows, we demonstrate the value of this robustness empirically.


\begin{table*}[t]
    \centering
  \begin{tabular}{p{1.0cm} p{0.6cm}p{0.6cm}p{0.6cm}p{0.6cm}p{0.6cm}p{0.6cm}p{0.6cm}p{0.6cm}p{0.6cm}p{0.6cm}p{0.6cm}p{0.6cm}p{0.6cm}}
    \toprule
    \multirow{2}{*}{Features} &
      \multicolumn{3}{c}{Lime (\%)} &
      \multicolumn{3}{c}{Responsibility (\%)} &
      \multicolumn{3}{c}{Holler-Packel (\%)} &
      \multicolumn{3}{c}{Deegan-Packel (\%)} \\
      & {1st} & {2nd} & {3rd}  & {1st} & {2nd} & {3rd}  & {1st} & {2nd} & {3rd}  & {1st} & {2nd} & {3rd}   \\
      \midrule
 Race  & 0.0 & 0.0 & 0.0 & 0.921 & 0.079 & 0.0 & 0.845 & 0.148 & 0.007 & 0.845 & 0.148 & 0.007 \\ 
UC1  & 0.492 & 0.508 & 0.0 & 0.601 & 0.399 & 0.0 & 0.157 & 0.843 & 0.0 & 0.157 & 0.843 & 0.0 \\ 
UC2  & 0.508 & 0.492 & 0.0 & 0.601 & 0.399 & 0.0 & 0.157 & 0.843 & 0.0 & 0.157 & 0.843 & 0.0 \\ 
    \bottomrule
  \end{tabular}
   \caption{This table shows the results of the LIME attack experiment on the Compas dataset. Each row represents the frequency of occurrence of either a sensitive feature (\emph{Race}) or an uncorrelated feature (\emph{UC1},\emph{UC2}) in the top 3 positions when ranked based on their LIME scores, Responsibility indices, Holler-Packel indices, and Deegan-Packel indices.
LIME explanations do not uncover the underlying biases of the attack model, whereas the Responsibility index, Deegan-Packel index, and Holler-Packel index successfully uncover the underlying biases of the attack model in the explanations they generate.
}
    \label{tab:lime_compas2}   
\end{table*}

\begin{table*}[t]
    \centering
  \begin{tabular}{p{1.0cm} p{0.6cm}p{0.6cm}p{0.6cm}p{0.6cm}p{0.6cm}p{0.6cm}p{0.6cm}p{0.6cm}p{0.6cm}p{0.6cm}p{0.6cm}p{0.6cm}p{0.6cm}}
    \toprule
    \multirow{2}{*}{Features} &
      \multicolumn{3}{c}{SHAP (\%)} &
      \multicolumn{3}{c}{Responsibility (\%)} &
      \multicolumn{3}{c}{Holler-Packel (\%)} &
      \multicolumn{3}{c}{Deegan-Packel (\%)} \\
      & {1st} & {2nd} & {3rd}  & {1st} & {2nd} & {3rd}  & {1st} & {2nd} & {3rd}  & {1st} & {2nd} & {3rd}   \\
      \midrule
   Race  & 0.416 & 0.238 & 0.141 & 0.946 & 0.044 & 0.01 & 0.867 & 0.036 & 0.052 & 0.867 & 0.039 & 0.057 \\ 
UC1  & 0.252 & 0.249 & 0.172 & 0.608 & 0.316 & 0.067 & 0.146 & 0.47 & 0.215 & 0.146 & 0.552 & 0.138 \\ 
UC2  & 0.215 & 0.249 & 0.304 & 0.618 & 0.297 & 0.08 & 0.148 & 0.466 & 0.213 & 0.148 & 0.554 & 0.133 \\ 
    \bottomrule
  \end{tabular}
   \caption{This table shows the results of the SHAP attack experiment on the Compas dataset. Each row represents the frequency of occurrence of either a sensitive feature (\emph{Race}) or an uncorrelated feature (\emph{UC1},\emph{UC2}) in the top 3 positions when ranked based on their SHAP scores, Responsibility indices, Holler-Packel indices, and Deegan-Packel indices.
   }
    \label{tab:shap_compas2}   
\end{table*}

\begin{table*}[t]
    \centering
  \begin{tabular}{p{1.0cm} p{0.6cm}p{0.6cm}p{0.6cm}p{0.6cm}p{0.6cm}p{0.6cm}p{0.6cm}p{0.6cm}p{0.6cm}p{0.6cm}p{0.6cm}p{0.6cm}p{0.6cm}}
    \toprule
    \multirow{2}{*}{Features} &
      \multicolumn{3}{c}{Lime (\%)} &
      \multicolumn{3}{c}{Responsibility (\%)} &
      \multicolumn{3}{c}{Holler-Packel (\%)} &
      \multicolumn{3}{c}{Deegan-Packel (\%)} \\
      & {1st} & {2nd} & {3rd}  & {1st} & {2nd} & {3rd}  & {1st} & {2nd} & {3rd}  & {1st} & {2nd} & {3rd}   \\
      \midrule
  Gender  & 0.0 & 1.0 & 0.0 & 1.0 & 0.0 & 0.0 & 1.0 & 0.0 & 0.0 & 1.0 & 0.0 & 0.0 \\ 
LR  & 1.0 & 0.0 & 0.0 & 0.46 & 0.54 & 0.0 & 0.0 & 0.69 & 0.31 & 0.0 & 0.72 & 0.28 \\ 
    \bottomrule
  \end{tabular}
   \caption{This table shows the results of the LIME attack experiment on the German Credit dataset. Each row represents the frequency of occurrence of either a sensitive feature (\emph{Gender}) or an uncorrelated feature (\emph{LoanRateAsPercentOfIncome}) in the top 3 positions when ranked based on their LIME scores, Responsibility indices, Holler-Packel indices, and Deegan-Packel indices.
   }
    \label{tab:lime_german1}   
\end{table*}

\begin{table*}[t]
    \centering
  \begin{tabular}{p{1.0cm} p{0.6cm}p{0.6cm}p{0.6cm}p{0.6cm}p{0.6cm}p{0.6cm}p{0.6cm}p{0.6cm}p{0.6cm}p{0.6cm}p{0.6cm}p{0.6cm}p{0.6cm}}
    \toprule
    \multirow{2}{*}{Features} &
      \multicolumn{3}{c}{SHAP (\%)} &
      \multicolumn{3}{c}{Responsibility (\%)} &
      \multicolumn{3}{c}{Holler-Packel (\%)} &
      \multicolumn{3}{c}{Deegan-Packel (\%)} \\
      & {1st} & {2nd} & {3rd}  & {1st} & {2nd} & {3rd}  & {1st} & {2nd} & {3rd}  & {1st} & {2nd} & {3rd}   \\
      \midrule
Gender  & 0.0 & 0.41 & 0.01 & 0.93 & 0.04 & 0.03 & 0.87 & 0.07 & 0.02 & 0.87 & 0.07 & 0.02 \\ 
LR  & 0.96 & 0.0 & 0.04 & 0.55 & 0.44 & 0.01 & 0.17 & 0.81 & 0.02 & 0.17 & 0.82 & 0.0 \\ 
    \bottomrule
  \end{tabular}
   \caption{This table shows the results of the SHAP attack experiment on the German Credit dataset. Each row represents the frequency of occurrence of either a sensitive feature (\emph{Gender}) or an uncorrelated feature (\emph{LoanRateAsPercentOfIncome}) in the top 3 positions when ranked based on their SHAP scores, Responsibility indices, Holler-Packel indices, and Deegan-Packel indices.
   }
    \label{tab:shap_german1}   
\end{table*}

\section{Empirical Evaluation}\label{sec:expts}
To showcase the robustness of the explanations generated by our methods, we study their empirical behavior against adversarial attacks proposed by \citet{Slack2020Shap}. 
Specifically, we investigate if our framework successfully uncovers underlying biases in adversarial classifiers that popular explanation methods like LIME and SHAP often fail to identify \citep{Slack2020Shap}. 
We describe the details of the datasets used in our experiments below.
\begin{description}[leftmargin=0cm]
    \item[Compas \citep{angwin2016bias}:] This dataset contains information about the demographics, criminal records, and Compas risk scores of 6172 individual defendants from Broward County, Florida. 
    Individuals are labeled with either a `high' or `low' risk score, with race as the sensitive feature.
    \item[German Credit \citep{Dua2019UCI}:] This dataset contains financial and demographic information on 1000 loan applicants.
    Each candidate is labeled as either a good or bad loan candidate. The sensitive feature is gender.
\end{description}

\subsection{Attack Model}
\begin{figure}[t]
    \centering
    \begin{tikzpicture}[->,rect/.style={rectangle, rounded corners, minimum width=2cm, minimum height=1cm,text centered, draw=black}, mycirc2/.style={circle,fill=white,minimum size=0.75cm,inner sep = 3pt},edge/.style = {->,> = stealth'}]
    \node[mycirc2] (poi) {$\vec x$};
    \node[rect, right=0.5cm of poi] (ood) {OOD Classifier};
    \node[rect, above=1cm of ood, xshift=2cm] (unb) {Unbiased Classifier};
    \node[rect, below=1cm of ood, xshift=2cm] (b) {Biased Classifier};
    \node[mycirc2,right=0.5cm of unb] (ounb) {Output};
    \node[mycirc2,right=0.5cm of b] (ob){Output};
    \draw[edge] (poi) -- (ood);
    \draw[edge] (ood) -- (unb) node[midway, left=5pt] {If $\vec x$ is OOD};
    \draw[edge] (ood) -- (b) node[midway, left=5pt] {If $\vec x$ is \emph{not} OOD};
    \draw[edge] (b) -- (ob);
    \draw[edge] (unb) -- (ounb);
    \end{tikzpicture}
    \caption{A pictorial description of the attack model. OOD is short for out-of-distribution.}
    \label{fig:attack-model}
\end{figure}
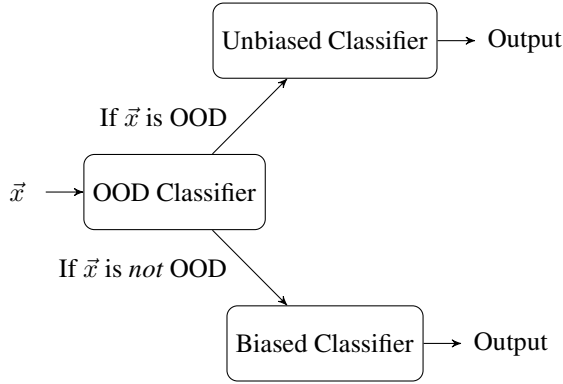
We evaluate the robustness of our explanation methods using the adversarial attacks proposed by \citet{Slack2020Shap} for LIME and SHAP. 
The underlying attack model is a two-level adversarial classifier in both adversarial attacks. The first level of the adversarial classifier is an out-of-distribution (OOD) classifier that predicts if a point is OOD or not. The second level of the adversarial classifier consists of a biased and unbiased prediction model, both of which predict the target label for a given data point. The biased prediction model makes predictions solely based on the sensitive feature in the dataset. In contrast, the unbiased prediction model makes predictions using features that are uncorrelated with the sensitive feature.

With the Compas dataset, the biased model uses the \emph{Race} feature for predicting the target label. In contrast, the unbiased model uses either one or two uncorrelated synthetic features (\emph{UC1, UC2}) for predicting the label. When two uncorrelated features are used, the label is their values' \emph{XOR}.
Similarly, with the German Credit dataset, the unbiased and biased models use the \emph{LoanRateAsPercentOfIncome} feature and \emph{Gender} feature for predicting the target label, respectively.

For a given data point, the adversarial classifier first uses the out-of-distribution (OOD) classifier to predict whether the given data point is out-of-distribution. If the given data point is out of distribution, the adversarial classifier uses the unbiased prediction model to predict the target label; else, the adversarial classifier uses the biased classifier to predict the target label (see Figure \ref{fig:attack-model}).
Most points in the dataset are classified as in-distribution and therefore, the prediction of the attack model for these points will be made solely using the sensitive feature of the dataset.
Since the type of explanations generated by popular methods like LIME and SHAP  tend to be heavily influenced by the predictions of the classifier model on out-of-distribution data points, this attack is designed to hide the underlying biases of the model by ensuring the bias is only applied to in-distribution data points. 
For each method (SHAP and LIME), \citet{Slack2020Shap} use a different attack model with the same high level approach described above.
We relegate the specific implementation details of each attack model to Appendix \ref{app:experiments}.

\subsection{Experimental Setup}

We split a given dataset into train and test datasets in all our experiments. 
We use the training dataset to train out-of-distribution (OOD) classifiers for the LIME and SHAP attacks and the test dataset to evaluate our methods' robustness.
To generate explanations using our proposed abductive explanation aggregators, we must first compute the set of all abductive explanations for the adversarial classifier model. We do this using the MARCO algorithm \citep{LiffitonPMM16}. 
After generating the complete set of abductive explanations for the adversarial classifier, we compute the feature importance scores using each of our methods --- the Holler-Packel index, Deegan-Packel index, and the Responsibility index. We use these feature importance scores as explanations for each point in the test dataset.

We compare our methods with LIME and SHAP, computed using their respective publicly available libraries \citep{Lundberg2017Unified, Ribeiro2016Lime}. Code for reproducing the results can be found at  \href{https://shorturl.at/tJT09}{https://shorturl.at/tJT09}.

\subsection{Evaluating Robustness to Adversarial LIME and SHAP attacks}

 For each data point in the test dataset, we rank features based on the feature importance scores given by each explanation method. Note that we allow multiple features to hold the same rank if they have the same importance scores.  
 For each explanation method, we compute the fraction of data points in which the sensitive and uncorrelated features appear in the top three positions.  
 Since most of the points in the test dataset are `in-distribution' and classified as such by the OOD classifier, any good explanation method should identify that the adversarial classifier makes its prediction largely based on the sensitive feature for most of the points in the test dataset. 
 In other words, the sensitive feature should receive a high importance score.

\Cref{tab:lime_compas2} 
shows the percentage of data points for which the sensitive attribute (i.e., \emph{Race}) and the uncorrelated features (\emph{UC1} and \emph{UC2}) appear in the top three positions when features are ranked using LIME and our methods in the LIME attack experiment on the Compas dataset. While \Cref{tab:lime_compas2} presents results when two uncorrelated synthetic features (\emph{UC1, UC2}) are used in the unbiased model of the adversarial classifier, \Cref{tab:lime_compas1} in Appendix \ref{app:experiments} presents results when a single uncorrelated feature is used in the unbiased model of the adversarial classifier.

Similarly, \Cref{tab:shap_compas2} shows the percentage of data points for which the sensitive attribute (i.e., \emph{Race}) and the uncorrelated features (\emph{UC1} and \emph{UC2}) appear in the top three positions when features are ranked using SHAP and our methods in the SHAP attack experiment for the Compas dataset. Again, \Cref{tab:shap_compas2} presents results when two uncorrelated features are used in the unbiased model of the adversarial classifier and \Cref{tab:shap_compas1} in Appendix \ref{app:experiments} presents results when a single uncorrelated feature is used in the unbiased model of the adversarial classifier.
%


Since the biased classifier is used to predict the label for almost all the test points, we expect the explanations to assign a high feature importance score to the sensitive feature. However, we observe that in the LIME attack experiment, LIME does not always assign high scores to the sensitive feature --- \emph{Race} --- due to which \emph{Race} does not at all appear in the top three positions when two uncorrelated features are used. The uncorrelated features are incorrectly ranked higher than the sensitive feature. 
On the other hand, the Responsibility index, the Holler-Packel index, and the Deegan-Packel index assign the highest feature importance scores to \emph{Race}: \emph{Race} appears in the top position for the majority of the instances ($>84\%$). 
It is important to note that the instances in which our explanation methods do not assign a high importance score to the \emph{Race} feature are the instances where the OOD classifier classifies test dataset instances as out-of-distribution instances. 
We observe a similar pattern to LIME in the SHAP attack experiment. In this experiment, abductive explanation aggregators rank {\em Race} as the most important feature in at least $86\%$ of test data, whereas SHAP ranks \emph{Race} as the most important feature only for $41.6 \%$ of the returned explanations.

We see similar results with the German Credit dataset 
reported in \Cref{tab:lime_german1} and \Cref{tab:shap_german1}.
In both LIME and SHAP attacks, we observe that the \emph{LoanRateAsPercentOfIncome} feature appears in the top position for most of the delivered explanations. However, the sensitive feature --- \emph{Gender} --- does not appear in the top position in any instance. 

In contrast, the Responsibility Index, the Holler-Packel Index, and the Deegan-Packel Index correctly assign the highest feature importance score to the sensitive feature --- \emph{Gender} --- for most of the data points; \emph{Gender} appears in the top position in $>87\%$ of the instances in both the LIME and SHAP attack experiments.
Clearly, we can conclude that our abductive explanation aggregators generate more robust and reliable explanations to adversarial attacks than LIME and SHAP.

\section{Conclusion and Future Work}
In this work, we aggregate abductive explanations into feature importance scores. 
We present three methods that aggregate abductive explanations, showing that each of them uniquely satisfies a set of desirable properties. 
We also empirically evaluate each of our methods, showing that they are robust to attacks that SHAP and LIME are vulnerable to.

At a higher level, our work combines satisfiability theory and cooperative game theory to explain the decisions of machine learning models. We do so using the well-studied concept of abductive explanations. However, our framework can potentially be extended to other explanation concepts from satisfiability theory as well, such as {\em contrastive explanations} \citep{IgnatievNA020} and {\em probabilistic abductive explanations} \citep{izza-ijar23}; this is an important area for future work. 

Our focus in this paper has been the axiomatic characterization and comparison of different measures. We believe an empirical comparison of the three methods we propose is also worth exploring in future work. This study is likely to yield insights into the differences in applicability of each of our three methods, further leading to a deeper understanding into how abductive explanations should be aggregated.

\paragraph{Acknowledgments.}
This research supported in part by the National Research Foundation, Prime Minister’s Office, Singapore under its Campus for Research Excellence and Technological Enterprise (CREATE) program.

\bibliographystyle{plainnat}
\bibliography{abb,explanations}

\begin{thebibliography}{40}
\providecommand{\natexlab}[1]{#1}
\providecommand{\url}[1]{\texttt{#1}}
\expandafter\ifx\csname urlstyle\endcsname\relax
  \providecommand{\doi}[1]{doi: #1}\else
  \providecommand{\doi}{doi: \begingroup \urlstyle{rm}\Url}\fi

\bibitem[Angwin et~al.(2016)Angwin, Larson, Mattu, and
  Kirchner]{angwin2016bias}
Julia Angwin, Jeff Larson, Surya Mattu, and Lauren Kirchner.
\newblock Machine bias: There’s software used across the country to predict
  future criminals. and it’s biased against blacks.
\newblock \emph{ProPublica}, May 2016.

\bibitem[Arenas et~al.(2022)Arenas, Barcel{\'{o}}, Orth, and
  Subercaseaux]{ArenasBOS22}
Marcelo Arenas, Pablo Barcel{\'{o}}, Miguel A.~Romero Orth, and Bernardo
  Subercaseaux.
\newblock On computing probabilistic explanations for decision trees.
\newblock In \emph{Proceedings of the 35th Annual Conference on Neural
  Information Processing Systems (NeurIPS)}, 2022.

\bibitem[Audemard et~al.(2022)Audemard, Bellart, Bounia, Koriche, Lagniez, and
  Marquis]{audemard2022preferred}
Gilles Audemard, Steve Bellart, Louenas Bounia, Frederic Koriche, Jean-Marie
  Lagniez, and Pierre Marquis.
\newblock On preferred abductive explanations for decision trees and random
  forests.
\newblock In \emph{Proceedings of the 31st International Joint Conference on
  Artificial Intelligence (IJCAI)}, pages 643--650, 2022.

\bibitem[Barocas et~al.(2020)Barocas, Selbst, and
  Raghavan]{barocas2020explanations}
Solon Barocas, Andrew~D. Selbst, and Manish Raghavan.
\newblock The hidden assumptions behind counterfactual explanations and
  principal reasons.
\newblock In \emph{Proceedings of the 3rd ACM Conference on Fairness,
  Accountability, and Transparency (FAccT)}, pages 80--89, 2020.

\bibitem[Chalkiadakis et~al.(2011)Chalkiadakis, Elkind, and
  Wooldridge]{Wooldridge2011}
Georgios Chalkiadakis, Edith Elkind, and Michael Wooldridge, editors.
\newblock \emph{Computational Aspects of Cooperative Game Theory}.
\newblock Morgan \& Claypool Publishers, 1st edition, 2011.

\bibitem[Chockler and Halpern(2004)]{chockler2004responsibility}
Hana Chockler and Joseph~Y Halpern.
\newblock Responsibility and blame: A structural-model approach.
\newblock \emph{Journal of Artificial Intelligence Research}, 22:\penalty0
  93--115, 2004.

\bibitem[Datta et~al.(2016)Datta, Sen, and Zick]{Datta2016Qii}
Anupam Datta, S.~Sen, and Yair Zick.
\newblock Algorithmic transparency via quantitative input influence: Theory and
  experiments with learning systems.
\newblock \emph{IEEE Symposium on Security and Privacy}, pages 598--617, 2016.

\bibitem[Deegan and Packel(1978)]{DeeganPackel1978}
J.~Deegan and Edward Packel.
\newblock A new index of power for simple n-person games.
\newblock \emph{International Journal of Game Theory}, 7:\penalty0 113--123,
  1978.

\bibitem[Dua and Graff(2017)]{Dua2019UCI}
Dheeru Dua and Casey Graff.
\newblock {UCI} machine learning repository, 2017.
\newblock URL \url{http://archive.ics.uci.edu/ml}.

\bibitem[Frye et~al.(2020)Frye, Rowat, and Feige]{Frye2020Assymetric}
Christopher Frye, Colin Rowat, and Ilya Feige.
\newblock Asymmetric shapley values: Incorporating causal knowledge into
  model-agnostic explainability.
\newblock In \emph{Proceedings of the 34th Annual Conference on Neural
  Information Processing Systems (NeurIPS)}, 2020.

\bibitem[Galhotra et~al.(2021)Galhotra, Pradhan, and
  Salimi]{Galhotra2021Contrastive}
Sainyam Galhotra, Romila Pradhan, and Babak Salimi.
\newblock Explaining black-box algorithms using probabilistic contrastive
  counterfactuals, 2021.

\bibitem[Holler(1982)]{Holler1982OG}
Manfred~J. Holler.
\newblock Forming coalitions and measuring voting power.
\newblock \emph{Political Studies}, 30:\penalty0 262--271, 1982.

\bibitem[Holler and Packel(1983)]{HollerPackel1983}
Manfred~J. Holler and Edward~W. Packel.
\newblock Power, luck and the right index.
\newblock \emph{Journal of Economics}, 43:\penalty0 21--29, 1983.

\bibitem[Huang and Marques-Silva(2023)]{huang2023inadequacy}
Xuanxiang Huang and Joao Marques-Silva.
\newblock The inadequacy of shapley values for explainability, 2023.

\bibitem[Ignatiev et~al.(2019{\natexlab{a}})Ignatiev, Narodytska, and
  Marques-Silva]{ignatiev2019abductiveog}
Alexey Ignatiev, Nina Narodytska, and Joao Marques-Silva.
\newblock Abduction-based explanations for machine learning models.
\newblock In \emph{Proceedings of the 33rd AAAI Conference on Artificial
  Intelligence (AAAI)}, 2019{\natexlab{a}}.

\bibitem[Ignatiev et~al.(2019{\natexlab{b}})Ignatiev, Narodytska, and
  Marques{-}Silva]{inms-corr19}
Alexey Ignatiev, Nina Narodytska, and Joao Marques{-}Silva.
\newblock On validating, repairing and refining heuristic {ML} explanations.
\newblock \emph{CoRR}, abs/1907.02509, 2019{\natexlab{b}}.

\bibitem[Ignatiev et~al.(2020{\natexlab{a}})Ignatiev, Narodytska, Asher, and
  Marques{-}Silva]{IgnatievNA020}
Alexey Ignatiev, Nina Narodytska, Nicholas Asher, and Jo{\~{a}}o
  Marques{-}Silva.
\newblock From contrastive to abductive explanations and back again.
\newblock In Matteo Baldoni and Stefania Bandini, editors, \emph{Proceedings of
  the 19thInternational Conference of the Italian Association for Artificial
  Intelligence (AIxIA)}, volume 12414, pages 335--355. Springer,
  2020{\natexlab{a}}.

\bibitem[Ignatiev et~al.(2020{\natexlab{b}})Ignatiev, Narodytska, and
  Marques{-}Silva]{inms-rcra20}
Alexey Ignatiev, Nina Narodytska, and Joao Marques{-}Silva.
\newblock On formal reasoning about explanations.
\newblock In \emph{RCRA}, 2020{\natexlab{b}}.

\bibitem[Ignatiev et~al.(2022)Ignatiev, Izza, Stuckey, and
  Marques-Silva]{silva2022treeensembles}
Alexey Ignatiev, Yacine Izza, Peter~J. Stuckey, and Joao Marques-Silva.
\newblock Using maxsat for efficient explanations of tree ensembles.
\newblock In \emph{Proceedings of the 36th AAAI Conference on Artificial
  Intelligence (AAAI)}, pages 3776--3785, 2022.

\bibitem[Izza et~al.(2023)Izza, Huang, Ignatiev, Narodytska, Cooper, and
  Marques{-}Silva]{izza-ijar23}
Yacine Izza, Xuanxiang Huang, Alexey Ignatiev, Nina Narodytska, Martin~C.
  Cooper, and Jo{\~{a}}o Marques{-}Silva.
\newblock On computing probabilistic abductive explanations.
\newblock \emph{Int. J. Approx. Reason.}, 159, 2023.

\bibitem[Liffiton et~al.(2016)Liffiton, Previti, Malik, and
  Marques{-}Silva]{LiffitonPMM16}
Mark~H. Liffiton, Alessandro Previti, Ammar Malik, and Jo{\~{a}}o
  Marques{-}Silva.
\newblock Fast, flexible {MUS} enumeration.
\newblock \emph{Constraints An Int. J.}, 21\penalty0 (2):\penalty0 223--250,
  2016.

\bibitem[Lundberg and Lee(2017)]{Lundberg2017Unified}
Scott~M. Lundberg and Su-In Lee.
\newblock A unified approach to interpreting model predictions.
\newblock In \emph{Proceedings of the 31st Annual Conference on Neural
  Information Processing Systems (NeurIPS)}, pages 4768--4777, 2017.

\bibitem[Marques-Silva et~al.(2020)Marques-Silva, Gerspacher, Cooper, Ignatiev,
  and Narodytska]{silva2020naivebayes}
Joao Marques-Silva, Thomas Gerspacher, Martin~C. Cooper, Alexey Ignatiev, and
  Nina Narodytska.
\newblock Explaining naive bayes and other linear classifiers with polynomial
  time and delay.
\newblock In \emph{Proceedings of the 34th Annual Conference on Neural
  Information Processing Systems (NeurIPS)}, 2020.

\bibitem[Marques-Silva et~al.(2021)Marques-Silva, Gerspacher, Cooper, Ignatiev,
  and Narodytska]{silva2021monotonic}
Joao Marques-Silva, Thomas Gerspacher, Martin~C Cooper, Alexey Ignatiev, and
  Nina Narodytska.
\newblock Explanations for monotonic classifiers.
\newblock In \emph{Proceedings of the 38th International Conference on Machine
  Learning (ICML)}, pages 7469--7479, 2021.

\bibitem[Ozbayoglu et~al.(2020)Ozbayoglu, Gudelek, and
  Sezer]{ozbayoglu2020finance}
Ahmet~Murat Ozbayoglu, Mehmet~Ugur Gudelek, and Omer~Berat Sezer.
\newblock Deep learning for financial applications : A survey.
\newblock \emph{Applied Soft Computing}, 93:\penalty0 106384, 2020.
\newblock ISSN 1568-4946.

\bibitem[Pandey et~al.(2022)Pandey, {Kumar Pandey}, {Pratap Mishra}, and
  Rhmann]{pandey2022healthcare}
Babita Pandey, Devendra {Kumar Pandey}, Brijendra {Pratap Mishra}, and Wasiur
  Rhmann.
\newblock A comprehensive survey of deep learning in the field of medical
  imaging and medical natural language processing: Challenges and research
  directions.
\newblock \emph{Journal of King Saud University - Computer and Information
  Sciences}, 34:\penalty0 5083--5099, 2022.
\newblock ISSN 1319-1578.

\bibitem[Patel et~al.(2021)Patel, Strobel, and Zick]{Patel2021Bii}
Neel Patel, Martin Strobel, and Yair Zick.
\newblock High dimensional model explanations: An axiomatic approach.
\newblock In \emph{Proceedings of the 4th ACM Conference on Fairness,
  Accountability, and Transparency (FAccT)}, pages 401--411. {ACM}, 2021.

\bibitem[Pedregosa et~al.(2011)Pedregosa, Varoquaux, Gramfort, Michel, Thirion,
  Grisel, Blondel, Prettenhofer, Weiss, Dubourg, Vanderplas, Passos,
  Cournapeau, Brucher, Perrot, and Duchesnay]{scikit-learn}
F.~Pedregosa, G.~Varoquaux, A.~Gramfort, V.~Michel, B.~Thirion, O.~Grisel,
  M.~Blondel, P.~Prettenhofer, R.~Weiss, V.~Dubourg, J.~Vanderplas, A.~Passos,
  D.~Cournapeau, M.~Brucher, M.~Perrot, and E.~Duchesnay.
\newblock Scikit-learn: Machine learning in {P}ython.
\newblock \emph{Journal of Machine Learning Research}, 12:\penalty0 2825--2830,
  2011.

\bibitem[Qayyum et~al.(2021)Qayyum, Qadir, Bilal, and
  Al-Fuqaha]{qayyum2021healthcare}
Adnan Qayyum, Junaid Qadir, Muhammad Bilal, and Ala Al-Fuqaha.
\newblock Secure and robust machine learning for healthcare: A survey.
\newblock \emph{IEEE Reviews in Biomedical Engineering}, 14:\penalty0 156--180,
  2021.

\bibitem[Ribeiro et~al.(2016)Ribeiro, Singh, and Guestrin]{Ribeiro2016Lime}
Marco~Tulio Ribeiro, Sameer Singh, and Carlos Guestrin.
\newblock "why should i trust you?": Explaining the predictions of any
  classifier.
\newblock In \emph{Proceedings of the 22nd International Conference on
  Knowledge Discovery and Data Mining (KDD)}, page 1135–1144, 2016.

\bibitem[Rudin(2019)]{rudin2019stop}
Cynthia Rudin.
\newblock Stop explaining black box machine learning models for high stakes
  decisions and use interpretable models instead.
\newblock \emph{Nature machine intelligence}, 1\penalty0 (5):\penalty0
  206--215, 2019.

\bibitem[Shapley(1953)]{Shapley1953OG}
Lloyd~S. Shapley.
\newblock \emph{A Value for n-Person Games}, pages 307--318.
\newblock Princeton University Press, 1953.

\bibitem[Shih et~al.(2018)Shih, Choi, and Darwiche]{shih2018abductive}
Andy Shih, Arthur Choi, and Adnan Darwiche.
\newblock A symbolic approach to explaining bayesian network classifiers.
\newblock In \emph{Proceedings of the 27th International Joint Conference on
  Artificial Intelligence (IJCAI)}, page 5103–5111, 2018.

\bibitem[Slack et~al.(2020)Slack, Hilgard, Jia, Singh, and
  Lakkaraju]{Slack2020Shap}
Dylan Slack, Sophie Hilgard, Emily Jia, Sameer Singh, and Himabindu Lakkaraju.
\newblock Fooling lime and shap: Adversarial attacks on post hoc explanation
  methods.
\newblock In \emph{Proceedings of the 3rd AAAI/ACM Conference on Artifical
  Intelligence, Ethics, and Society (AIES)}, 2020.

\bibitem[Sundararajan and Najmi(2020)]{Sundarajan2020Shapley}
Mukund Sundararajan and Amir Najmi.
\newblock The many shapley values for model explanation.
\newblock In \emph{Proceedings of the 37th International Conference on Machine
  Learning (ICML)}, pages 9269--9278, 2020.

\bibitem[Sundararajan et~al.(2017)Sundararajan, Taly, and
  Yan]{Sundarajan2017NeuralNetworks}
Mukund Sundararajan, Ankur Taly, and Qiqi Yan.
\newblock Axiomatic attribution for deep networks.
\newblock In \emph{Proceedings of the 34th International Conference on Machine
  Learning (ICML)}, page 3319–3328, 2017.

\bibitem[W{\"{a}}ldchen et~al.(2021)W{\"{a}}ldchen, MacDonald, Hauch, and
  Kutyniok]{waldchen2021probabilistic}
Stephan W{\"{a}}ldchen, Jan MacDonald, Sascha Hauch, and Gitta Kutyniok.
\newblock The computational complexity of understanding binary classifier
  decisions.
\newblock \emph{J. Artif. Intell. Res.}, 70:\penalty0 351--387, 2021.

\bibitem[Watson et~al.(2021)Watson, Gultchin, Taly, and
  Floridi]{watson2021local}
David Watson, Limor Gultchin, Ankur Taly, and Luciano Floridi.
\newblock Local explanations via necessity and sufficiency: Unifying theory and
  practice, 2021.

\bibitem[Young(1985)]{Young1985Shapley}
H.~Young.
\newblock Monotonic solutions of cooperative games.
\newblock \emph{International Journal of Game Theory}, 14:\penalty0 65--72,
  1985.

\bibitem[Yu et~al.(2023)Yu, Ignatiev, and Stuckey]{Ignatiev-corr23}
Jinqiang Yu, Alexey Ignatiev, and Peter~J. Stuckey.
\newblock On formal feature attribution and its approximation.
\newblock \emph{CoRR}, abs/2307.03380, 2023.

\end{thebibliography}

\newpage
\appendix

\section{Missing Proofs from Section \ref{sec:framework}}

\thmhollerpackelaxioms*
\begin{proof}
It is easy to see that the Holler-Packel index satisfies these properties so we go ahead and show uniqueness.

Let us denote the aggregator that satisfies these properties by $\gamma$ and let us use $(\vec x, f)$ to denote a datapoint-model pair. We show uniqueness via induction on $|\cal M(\vec x, f)|$. When $|\cal M(\vec x, f)| = 1$, let the only abductive explanation be $T$. For all $i$ in $T$, by Symmetry and $(\sum_{i \in N} |\cal M_i(x, f)|)$-Efficiency, we get $\gamma_i(\vec x, f) = 1$. For all $i \in N \setminus T$, by Null Feature, we get $\gamma_i(\vec x, f) = 0$. This coincides with the Holler-Packel index.

Now assume $|\cal M(\vec x, f)| = m$ i.e. $\cal M(\vec x, f) = \{S_1, S_2, \dots S_m\}$ for some sets $S_1, S_2, \dots, S_m$. Let $S = \bigcap_{j \in [m]} S_j$ be the set of of features where are present in all abductive explanations. For any $i \notin S$, let $(\vec y, g)$ be a datapoint-model pair such that $\cal M(\vec y, g) = \cal M_i(\vec x, f)$. Such a datapoint-model pair trivially exists. Since $|\cal M(\vec y, g)| < |\cal M(\vec x, f)|$, we can apply the inductive hypothesis and $\gamma_i(\vec y, g)$ coincides with the Holler-Packel index. Therefore $\gamma_i(\vec y, g) = |\cal M_i(\vec y, g)|$. Using Minimal Monotonicity, we get that $\gamma_i(\vec x, f) = \gamma_i(\vec y, g) = |\cal M_i(\vec x, f)|$ as well. Equality holds since $\cal M_i(\vec y, g) = \cal M_i(\vec x, f)$. Therefore, we get $\gamma_i(\vec x, f) = |\cal M_i(\vec x, f)|$ which coincides with the Holler-Packel index.

For all $i \in S$, using Symmetry, they all have the same value and using $(\sum_{i \in N} |\cal M_i(x, f)|)$-Efficiency, this value is unique and since all the other features coincide with the Holler-Packel index and the Holler-Packel index satisfies these axioms, it must be the case that $\gamma_i(\vec x, f)$ coincides with the Holler-Packel index as well for all $i \in S$.
\end{proof}

\thmdeeganpackelaxioms*
\begin{proof}
This proof is very similar to that of the Holler-Packel index (Theorem \ref{thm:holler-packel-axioms}).
It is easy to see that the Deegan-Packel index satisfies these properties so we go ahead and show uniqueness.

Let us denote the aggregator that satisfies these properties by $\gamma$ and let us use $(\vec x, f)$ to denote a datapoint-model pair. We show uniqueness via induction on $|\cal M(\vec x, f)|$. When $|\cal M(\vec x, f)| = 1$, let the only abductive explanation be $T$. For all $i$ in $T$, by Symmetry and $|\cal M(\vec x, f)|$-Efficiency, we get $\gamma_i(\vec x, f) = \frac{1}{|T|}$. For all $i \in N \setminus T$, by Null Feature, we get $\gamma_i(\vec x, f) = 0$. This coincides with the Deegan-Packel index.

Now assume $|\cal M(\vec x, f)| = m$ i.e. $\cal M(\vec x, f) = \{S_1, S_2, \dots S_m\}$ for some sets $S_1, S_2, \dots, S_m$. Let $S = \bigcap_{j \in [m]} S_j$ be the set of of features where are present in all abductive explanations. For any $i \notin S$, let $(\vec y, g)$ be a datapoint-model pair such that $\cal M(\vec y, g) = \cal M_i(\vec x, f)$. Such a datapoint-model pair trivially exists. Since $|\cal M(\vec y, g)| < |\cal M(\vec x, f)|$, we can apply the inductive hypothesis and $\gamma_i(\vec y, g)$ coincides with the Deegan-Packel index. Therefore $\gamma_i(\vec y, g) = \sum_{S \in \cal M_i(\vec y, g)} \frac1{|S|}$. Using Minimal Monotonicity, we get that $\gamma_i(\vec x, f) = \gamma_i(\vec y, g) = \sum_{S \in \cal M_i(\vec y, g)} \frac1{|S|}$ as well. Equality holds since $\cal M_i(\vec y, g) = \cal M_i(\vec x, f)$. Therefore, we get $\gamma_i(\vec x, f) = = \sum_{S \in \cal M_i(\vec x, f)} \frac1{|S|}$ which coincides with the Deegan-Packel index.

For all $i \in S$, using Symmetry, they all have the same value and using $|\cal M(\vec x, f)|$-Efficiency, this value is unique and since all the other features coincide with the Deegan-Packel index and the Deegan-Packel index satisfies these axioms, it must be the case that $\gamma_i(\vec x, f)$ coincides with the Deegan-Packel index as well for all $i \in S$.
\end{proof}

\thmresponsibilityindexaxioms*
\begin{proof}
It is easy to see that the responsibility index satisfies Minimum Size Monotonicity, Unit Efficiency, Null Feature and Symmetry. We show using the following Lemma that the responsibility index satisfies Contraction.

\begin{restatable}{lemma}{lemresponsibilitycontraction}\label{lem:responsibility-contraction}
The responsibility index $\rho(\vec x, f)$ satsifies Contraction.
\end{restatable}
\begin{proof}
For any set $T$ which does not contain a Null Feature, the responsibility index $\rho_{[T]}(\vec x^{[T]},f)$ is non-zero and corresponds to the inverse of the size of some set $S_T \in \cal M_{[T]}(\vec x^{[T]}, f)$. This implies that there must be some set $S$ in $\cal M(\vec x, f)$ which contains some non-empty subset $T' \subseteq T$ such that $S_T \setminus [T] = S \setminus T'$. This is obtained from the definition of a contraction. From the definition of responsibility index, we have $\rho_{[T]}(\vec x^{[T]}, f) = \frac1{k - |T'| + 1}$ where $k = |S|$.

We first show that the total responsibility index of the elements in $T'$ under $(\vec x, f)$ is weakly greater than the responsibility of $[T]$ under $(\vec x^{[T]}, f)$. 

Let $k$ be the size of $S$. Then, for all feasible $|T'|$ and $k$, we  have since $|T'| \in [1, k]$, the following inequality:
\begin{align}
    |T'|^2 - (k + 1)|T'| + k &\le 0 \notag \\
    \implies \frac{1}{k - |T'| + 1} &\le \frac{|T'|}{k} \notag \\
    \implies \rho_{[T]}(\vec x^{[T]}, f) &\le \sum_{i \in T'} \rho_i(\vec x, f) \label{eq:responsibility-tprime}
\end{align}

where \eqref{eq:responsibility-tprime} is true since the existence of $S$ gives us a lower bound of $\frac1k$ on the responsibility indices of all the elements in $T'$.
Since the responsibility index is always non-negative, from \eqref{eq:responsibility-tprime}, we have
\begin{align*}
\rho_{[T]}(\vec x^{[T]}, f) \le \sum_{i \in T} \rho_i(\vec x, f) 
\end{align*}
which is the first part of the Contraction property.

To show the second part, assume that $T' = T \in \cal M(\vec x, f)$ where none of the elements in $T$ are present in a smaller abductive explanation. They all have a responsibility of $1/|T|$. We have $\rho_{[T]}(\vec x^{[T]}, f) = 1$ since the set $\{[T]\} \in \cal M_{[T]}(\vec x^{[T]}, f)$. It is easy to see that this satisfies the equality condition in the Contraction Property.
\end{proof}
We now show uniqueness via induction on the size of $|\cal M(\vec x, f)|$. Let an arbitrary aggregator which satisfies the above properties be denoted by $\gamma(\vec x, f)$. When $|\cal M(\vec x, f)| = 1$, let $\cal M(\vec x, f) = \{T\}$. If $T= \{i\}$ for some $i \in N$, then by Unit Efficiency, $\gamma_i(\vec x, f) = 1$ and by Null Feature, $\gamma_{i'}(\vec x, f) = 0$ for $i' \ne i$. This coincides with the responsibility index.

When $|T| \ge 2$, using Contraction, we get $\gamma_{[T]}(\vec x^{[T]}, f) = \sum_{i \in T} \gamma_i(\vec x, f)$. Note that equality holds since $T$ is an abductive explanation and the smallest abductive explanation for all the elements in $T$. Using Unit Efficiency, we get $\gamma_{[T]}(\vec x^{[T]}, f) = 1$. Using symmetry, $\gamma_i(\vec x, f) = \gamma_j(\vec x, f)$ for all $i, j \in T$. Therefore $\gamma_i(\vec x, f) = \frac1{|T|}$ for all $i \in T$. For all $i \in N \setminus T$, $\gamma_i(\vec x, f) = 0$ because of the Null Feature property. This coincides with the responsibility index $\rho(\vec x, f)$ for all $i \in N$.

Now assume $|\cal M(\vec x, f)| = m$ i.e. $\cal M(\vec x, f) = \{S_1, S_2, \dots S_m\}$ for some sets $S_1, S_2, \dots, S_m$. 
Let $S$ be the set of of features which are present in at least one abductive explanation. 
For any $i \in N$, let $S_{i}$ be the smallest abductive explanation that $i$ is in (if there are multiple, we choose one arbitrarily). 
Let $(\vec y, g)$ be the datapoint-model pair such that $\cal M(\vec y, g) = S_i$. 
By Minimum Size Monotonicity, $\gamma_i(\vec x, f) = \gamma_i(\vec y, g)$. Note that equality holds since the smallest abductive explanations that contain $i$ have the same size in both $(\vec x, f)$ and $(\vec y, g)$. By the inductive hypothesis, $\gamma_i(\vec y, g)$ corresponds to the responsibility index for $i$ under $(\vec y, g)$. Therefore $\gamma_i(\vec x, f) = \gamma_i(\vec y, g) = 1/|S_i|$ for all $i \in S$. This coincides with the degree of responsibility, since $S_i$ is the smallest abductive explanation that contains $i$.

For all $i \notin S$, we have $\gamma_i(\vec x, f) = 0$ because of Null Feature and this coincides with the responsibility index as well.
\end{proof}

\propdeeganpackelimpossibility*
\begin{proof}
Consider a setting with 4 features $\{1, 2, 3, 4\}$. Assume for contradiction that there exists an aggregator $\gamma$ that satisfies these properties. Consider a datapoint-model pair $(\vec x, f)$ with $\cal M(\vec x, f) = \{\{1, 2\}\}$. Using Efficiency and Symmetry, we have that $\gamma_i(\vec x, f) = 1/2$ for all $i \in \{1, 2\}$. Now consider another datapoint-model pair $(\vec y, g)$ with $\cal M(\vec y, g) = \{\{1, 2\}, \{3, 4\}\}$. Then from Minimal Monotonicity, we have $\gamma_i(\vec y, g) = 1/2$ for all $i \in \{1, 2\}$. Similarly, $\gamma_j(\vec y, g) = 1/2$ for all $j \in \{3, 4\}$. However, this clearly violates efficiency since $\sum_{i \in N} \gamma_i(\vec y, g) = 2 \ne 1$. This is clearly a contradiction and therefore, such an aggregator cannot exist.
\end{proof}

\section{Algorithmic Loan Approval: an Example}\label{sec:example}
In this section, we discuss an example of algorithmic loan approval to show how the all the indices look like in practice. 
Consider a simple rule-based model $f$ trained on the features `Age', `Purpose', `Credit Score' and `Bank Balance'. The rule based-model has the following closed form expression:
\begin{align*}
    f(\vec x) = & (\text{Age} < 20 \land \text{Purpose} = \text{Education}) \\
    & \lor (\text{Age} > 30 \land \text{Purpose} = \text{Real Estate} \land \text{Credit > 700}) \\
    & \lor (\text{Credit} > 700 \land \text{Bank} > 300000) \\
    & \lor (\text{Age} > 25 \land \text{Bank} > 1000000)
\end{align*}
Let the point of interest $\vec x$ that we would like to explain be $(\text{Age} = 22, \text{Purpose} = \text{Real Estate}, \text{Credit} = 0, \text{Bank} = 50000)$.

Since $\vec x$ does not satisfy any of the rules, the model $f$ rejects the applicant; the abductive explanations of the outcome are 
$$\{(\text{Age}, \text{Credit}), (\text{Age}, \text{Bank}), (\text{Bank}, \text{Credit}, \text{Purpose})\}.$$
We compute the aggregators for all features, presented in Table \ref{tab:f-indices}. Table \ref{tab:f-indices} offers several interesting observations. 
All three indices have the same weak ordering over the set of features.
Age appears in two of three explanations, and all indices (weakly) rank Age as the most important feature;
however, the proportion of importance given to Age varies from index to index. 
On one hand, the responsibility index assigns Age the same importance as all other features (except for Purpose) as Age alone cannot change the outcome. 
On the other hand, the Deegan-Packel index assigns Age a strictly higher importance than all other features. 
We do not argue in favor of any index over another, but believe that they all provide useful insights about the output of $f$.
\begin{table}[h]
    \centering
    \begin{subtable}[h]{0.7\textwidth}
        \centering
        \begin{tabular}{|l | c | c | c| c|}
        \hline
        Index & Purpose & Age & Bank & Credit \\
        \hline
        Responsibility & 0.333 & 0.5 & 0.5 & 0.5\\
        Holler-Packel & 0.125 & 0.25 & 0.25 & 0.25\\
        Deegan-Packel & 0.042 & 0.125 & 0.104 & 0.104\\
        \hline
       \end{tabular}
       \caption{Index values explaining $f(\vec x)$}
       \label{tab:f-indices}
    \end{subtable}
    \hfill
    \begin{subtable}[h]{0.7\textwidth}
        \centering
        \begin{tabular}{|l | c | c | c| c|}
        \hline
        Index & Purpose & Age & Bank & Credit \\
        \hline
        Responsibility & 0 & 0.5 & 0.5 & 0.5\\
        Holler-Packel & 0 & 0.25 & 0.125 & 0.125\\
        Deegan-Packel & 0 & 0.125 & 0.062 & 0.062\\
        \hline
       \end{tabular}
       \caption{Index values explaining $g(\vec x)$}
       \label{tab:g-indices}
    \end{subtable}
     \caption{The explanations outputted for both $f(\vec x)$ (Table \ref{tab:f-indices}) and $g(\vec x)$ (Table \ref{tab:g-indices}), where $\vec x$ equals $(\text{Age} = 22, \text{Purpose} = \text{Real Estate}, \text{Credit} = 0, \text{Bank} = 50000)$. 
     }
     \label{tab:indices}
\end{table}

Another use of explanation indices is that they allow developers to compare different functions via the importance each feature has on the outcome. 
To show how this can be done, we create a new rule based function $g$ defined as follows:
\begin{align*}
    g(\vec x) = & (\text{Age} < 20 \land \text{Bank} > 25000) \\
    & \lor (\text{Bank} > 100000 \land \text{Credit} > 700) 
\end{align*}
The applicant $\vec x$ still does not satisfy any of the rules of $g$ and is rejected. 
However, the abductive explanations of $g(\vec x)$ --- $(\text{Age, Bank})$ and  $(\text{Age}, \text{Credit})$ --- are a subset of the abductive explanations of $f(\vec x)$. 
Ideally, the explanation indices should reflect this and assign features which are present in fewer causes less importance as compared to $f$. 
The indices explaining $g(\vec x)$ are presented in Table \ref{tab:g-indices}. 

The outputs are rather unsurprising. 
No index assigns a value to the Purpose since none of the abductive explanations contain it. 
However, even though the number of explanations containing Bank reduces, the responsibility index gives it the same amount of importance as $f$, while the other indices assigns it a lower importance than $f$. 

\section{On the Shapley Value}\label{apdx:shapley-value}
Recall that a cooperative game \citep{Wooldridge2011} is defined as a tuple $(N, v)$ where $N$ corresponds to a set of players and $v:2^N \rightarrow \R$ corresponds to the characteristic function of the game. $v(S)$ denotes the value of a set of players $S$; it can be thought of as the total money that the set of players $S$ will make if they work together. 

The Shapley Value \citep{Shapley1953OG, Young1985Shapley} assigns a score to each player in $N$ proportional to their importance in the cooperative. For each player $i \in N$, the Shapley value (denoted by $\phi(v)$) is defined as
\begin{align*}
    \phi_i(v) = \frac{1}{|N|!}\sum_{S \in N \setminus \{i\}} |S|! (|N| - |S| - 1)! (v(S + i) - v(S)).
\end{align*}
The Shapley value is the unique measure that satisfies the following four axioms:

\noindent \textbf{Monotonicity: } Let $v$ and $w$ be two value functions, and $i \in N$ be some player. If for all $S \subseteq N \setminus \{i\}$, we have $v(S \cup \{i\}) - v(S) \ge w(S \cup \{i\}) - w(S)$, then $\phi_i(v) \ge \phi_i(w)$.

\noindent \textbf{Symmetry (Shapley): } Let $v$ be a value function, and $i, j \in N$ be two players. If for all $S \subseteq N \setminus \{i, j\}$, we have $v(S \cup \{i\}) = v(S \cup \{j\})$, then $\phi_i(v) = \phi_j(v)$.

\noindent \textbf{Null Feature (Shapley): } Let $v$ be any value function and $i \in N$ be some player. If for all $S \subseteq N \setminus \{i\}$, $v(S \cup \{i\}) - v(S) = 0$, then $\phi_i(v) = 0$.

\noindent \textbf{Efficiency: } For any value function $v$, $\sum_{i \in N} \phi_i(v) = 1$.

Note immediately that the Shapley value is computed by studying the marginal contribution of a player $i$ to an arbitrary set $S$. This means, to compute the Shapley value, we will need sets other than the minimal winning sets (or the abductive explanations). The same can be said about the axioms Null Feature (Shapley) and Monotonicity. 


This rules the Shapley value out as an abductive explanation aggregator. It may however be possible to relax the definition of abductive explanations such that the Shapley value becomes a valid aggregator; we leave this question for future work.

\section{Additional Experimental Results and Details}\label{app:experiments}

\begin{table*}[t]
    \centering
  \begin{tabular}{p{0.5cm} p{0.9cm}p{0.9cm}p{0.9cm}p{0.9cm}p{0.9cm}p{0.9cm}p{0.9cm}p{0.9cm}p{0.9cm}p{0.9cm}p{0.9cm}p{0.9cm}p{0.9cm}}
    \toprule
    \multirow{2}{*}{Features} &
      \multicolumn{3}{c}{LIME (\%)} &
      \multicolumn{3}{c}{Responsibility (\%)} &
      \multicolumn{3}{c}{Holler-Packel (\%)} &
      \multicolumn{3}{c}{Deegan-Packel (\%)} \\
      & {1st} & {2nd} & {3rd}  & {1st} & {2nd} & {3rd}  & {1st} & {2nd} & {3rd}  & {1st} & {2nd} & {3rd}   \\
      \midrule
  Race  & 0.0 & 0.984 & 0.016 & 0.912 & 0.088 & 0.0 & 0.849 & 0.142 & 0.009 & 0.849 & 0.142 & 0.009 \\ 
UC1  & 1.0 & 0.0 & 0.0 & 0.567 & 0.433 & 0.0 & 0.151 & 0.849 & 0.0 & 0.151 & 0.849 & 0.0 \\ 
UC2  & 0.0 & 0.0 & 0.0 & 0.0 & 0.0 & 0.0 & 0.0 & 0.0 & 0.0 & 0.0 & 0.0 & 0.0 \\ 
    \bottomrule
  \end{tabular}
  \vspace{0.2cm}
   \caption{This table shows the results of the LIME attack experiment on the Compas dataset. Each row represents the frequency of occurrence of either a sensitive feature (\emph{Race}) or an uncorrelated feature (\emph{UC1},\emph{UC2}) in the top 3 positions when ranked based on their LIME scores, Responsibility indices, Holler-Packel indices, and Deegan-Packel indices.
   }
    \label{tab:lime_compas1}   
\end{table*}

\begin{table*}[t]
    \centering
  \begin{tabular}{p{0.5cm} p{0.9cm}p{0.9cm}p{0.9cm}p{0.9cm}p{0.9cm}p{0.9cm}p{0.9cm}p{0.9cm}p{0.9cm}p{0.9cm}p{0.9cm}p{0.9cm}p{0.9cm}}
    \toprule
    \multirow{2}{*}{Features} &
      \multicolumn{3}{c}{SHAP (\%)} &
      \multicolumn{3}{c}{Responsibility (\%)} &
      \multicolumn{3}{c}{Holler-Packel (\%)} &
      \multicolumn{3}{c}{Deegan-Packel (\%)} \\
      & {1st} & {2nd} & {3rd}  & {1st} & {2nd} & {3rd}  & {1st} & {2nd} & {3rd}  & {1st} & {2nd} & {3rd}   \\
      \midrule
   Race  & 0.199 & 0.476 & 0.058 & 0.949 & 0.044 & 0.007 & 0.864 & 0.052 & 0.061 & 0.864 & 0.074 & 0.049 \\ 
UC1  & 0.796 & 0.172 & 0.032 & 0.564 & 0.27 & 0.154 & 0.151 & 0.379 & 0.233 & 0.151 & 0.521 & 0.09 \\ 
UC2  & 0.0 & 0.0 & 0.002 & 0.028 & 0.379 & 0.439 & 0.003 & 0.251 & 0.4 & 0.003 & 0.241 & 0.085 \\ 
    \bottomrule
  \end{tabular}
    \vspace{0.2cm}
   \caption{This table shows the results of the SHAP attack experiment on the Compas dataset. Each row represents the frequency of occurrence of either a sensitive feature (\emph{Race}) or an uncorrelated feature (\emph{UC1},\emph{UC2}) in the top 3 positions when ranked based on their SHAP scores, Responsibility indices, Holler-Packel indices, and Deegan-Packel indices.
   }
    \label{tab:shap_compas1}  
    \vspace{-5mm}
\end{table*}

\subsection{Experimental Results with Different seeds}
\Cref{tab:lime_compas21}, \Cref{tab:shap_compas21}, \Cref{tab:lime_german11}, \Cref{tab:shap_german11},
\Cref{tab:lime_compas11}, and 
\Cref{tab:shap_compas11} show the statistically significant results computed for all our experiments. 
To obtain these results, we generate 10 variations of the Compas and German Credit datasets for each attack experiment using 10 seeds. For each explanation method and dataset, we report the mean and standard deviation of the frequency of occurrence of the sensitive feature and uncorrelated features in the top 3 positions
when ranked based on their feature importance scores.

\begin{table*}[ht]
    \centering
  \begin{tabular}{p{1.0cm} p{0.4cm}p{0.9cm}p{0.9cm}p{0.9cm}p{0.9cm}p{0.9cm}p{0.9cm}p{0.9cm}p{0.9cm}p{0.9cm}p{0.9cm}p{0.9cm}p{0.9cm}}
    \toprule
    \multirow{2}{*}{Features} &
      \multicolumn{3}{c}{Lime (\%)} &
      \multicolumn{3}{c}{Responsibility (\%)} &
      \multicolumn{3}{c}{Holler-Packel (\%)} &
      \multicolumn{3}{c}{Deegan-Packel (\%)} \\
      & {1st} & {2nd} & {3rd}  & {1st} & {2nd} & {3rd}  & {1st} & {2nd} & {3rd}  & {1st} & {2nd} & {3rd}   \\
      \midrule
Race & 0.0 ± 0.0& 0.0 ± 0.0& 0.09 ± 0.24& 0.91 ± 0.01& 0.09 ± 0.01& 0.0 ± 0.0& 0.83 ± 0.01& 0.16 ± 0.01& 0.02 ± 0.0& 0.83 ± 0.01& 0.16 ± 0.01& 0.02 ± 0.0 \\ 
UC1 & 0.49 ± 0.01& 0.51 ± 0.01& 0.0 ± 0.0& 0.59 ± 0.03& 0.41 ± 0.03& 0.0 ± 0.0& 0.17 ± 0.01& 0.83 ± 0.01& 0.0 ± 0.0& 0.17 ± 0.01& 0.83 ± 0.01& 0.0 ± 0.0 \\ 
UC2 & 0.51 ± 0.01& 0.49 ± 0.01& 0.0 ± 0.0& 0.59 ± 0.03& 0.41 ± 0.03& 0.0 ± 0.0& 0.17 ± 0.01& 0.83 ± 0.01& 0.0 ± 0.0& 0.17 ± 0.01& 0.83 ± 0.01& 0.0 ± 0.0 \\ 
    \bottomrule
  \end{tabular}
   \caption{This table shows the results of the LIME attack experiment on the Compas dataset. Each row represents the mean $\pm$ standard deviation of frequency of occurrence of either a sensitive feature (\emph{Race}) or an uncorrelated feature (\emph{UC1},\emph{UC2}) in the top 3 positions when ranked based on their LIME scores, Responsibility indices, Holler-Packel indices, and Deegan-Packel indices. The mean and standard deviation of frequency of occurence is computed over 10 datasets generated using 10 different seeds.
}
    \label{tab:lime_compas21}   
\end{table*}

\begin{table*}[ht]
    \centering
  \begin{tabular}{p{0.5cm} p{0.9cm}p{0.9cm}p{0.9cm}p{0.9cm}p{0.9cm}p{0.9cm}p{0.9cm}p{0.9cm}p{0.9cm}p{0.9cm}p{0.9cm}p{0.9cm}p{0.9cm}}
    \toprule
    \multirow{2}{*}{Features} &
      \multicolumn{3}{c}{SHAP (\%)} &
      \multicolumn{3}{c}{Responsibility (\%)} &
      \multicolumn{3}{c}{Holler-Packel (\%)} &
      \multicolumn{3}{c}{Deegan-Packel (\%)} \\
      & {1st} & {2nd} & {3rd}  & {1st} & {2nd} & {3rd}  & {1st} & {2nd} & {3rd}  & {1st} & {2nd} & {3rd}   \\
      \midrule
 Race & 0.24 ± 0.06& 0.23 ± 0.11& 0.09 ± 0.05& 0.93 ± 0.01& 0.06 ± 0.01& 0.01 ± 0.0& 0.84 ± 0.02& 0.05 ± 0.02& 0.06 ± 0.01& 0.84 ± 0.02& 0.06 ± 0.02& 0.07 ± 0.01 \\ 
UC1 & 0.29 ± 0.06& 0.15 ± 0.09& 0.11 ± 0.04& 0.67 ± 0.05& 0.26 ± 0.05& 0.07 ± 0.02& 0.17 ± 0.02& 0.43 ± 0.1& 0.23 ± 0.04& 0.17 ± 0.02& 0.5 ± 0.09& 0.13 ± 0.03 \\ 
UC2 & 0.29 ± 0.06& 0.18 ± 0.08& 0.13 ± 0.06& 0.64 ± 0.05& 0.27 ± 0.05& 0.08 ± 0.04& 0.17 ± 0.02& 0.41 ± 0.1& 0.23 ± 0.04& 0.17 ± 0.02& 0.47 ± 0.1& 0.13 ± 0.04 \\  
    \bottomrule
  \end{tabular}
   \caption{This table shows the results of the SHAP attack experiment on the Compas dataset. Each row represents the mean $\pm$ standard deviation of the frequency of occurrence of either a sensitive feature (\emph{Race}) or an uncorrelated feature (\emph{UC1},\emph{UC2}) in the top 3 positions when ranked based on their SHAP scores, Responsibility indices, Holler-Packel indices, and Deegan-Packel indices. The mean and standard deviation of frequency of occurrence is computed over 10 datasets generated using 10 different seeds.
   }
    \label{tab:shap_compas21}   
\end{table*}

\begin{table*}[ht]
    \centering
  \begin{tabular}{p{0.5cm} p{0.9cm}p{0.9cm}p{0.9cm}p{0.9cm}p{0.9cm}p{0.9cm}p{0.9cm}p{0.9cm}p{0.9cm}p{0.9cm}p{0.9cm}p{0.9cm}p{0.9cm}}
    \toprule
    \multirow{2}{*}{Features} &
      \multicolumn{3}{c}{Lime (\%)} &
      \multicolumn{3}{c}{Responsibility (\%)} &
      \multicolumn{3}{c}{Holler-Packel (\%)} &
      \multicolumn{3}{c}{Deegan-Packel (\%)} \\
      & {1st} & {2nd} & {3rd}  & {1st} & {2nd} & {3rd}  & {1st} & {2nd} & {3rd}  & {1st} & {2nd} & {3rd}   \\
      \midrule
Gender & 0.0 ± 0.0& 0.59 ± 0.3& 0.04 ± 0.08& 0.6 ± 0.33& 0.34 ± 0.24& 0.06 ± 0.14& 0.53 ± 0.39& 0.32 ± 0.22& 0.15 ± 0.21& 0.53 ± 0.39& 0.21 ± 0.23& 0.2 ± 0.18 \\ 
LR & 1.0 ± 0.0& 0.0 ± 0.0& 0.0 ± 0.0& 0.44 ± 0.28& 0.46 ± 0.18& 0.1 ± 0.14& 0.33 ± 0.35& 0.49 ± 0.26& 0.15 ± 0.2& 0.33 ± 0.35& 0.36 ± 0.35& 0.2 ± 0.16 \\ 
    \bottomrule
  \end{tabular}
   \caption{This table shows the results of the LIME attack experiment on the German Credit dataset. Each row represents the mean $\pm$ standard deviation of the frequency of occurrence of either a sensitive feature (\emph{Gender}) or an uncorrelated feature (\emph{LoanRateAsPercentOfIncome}) in the top 3 positions when ranked based on their LIME scores, Responsibility indices, Holler-Packel indices, and Deegan-Packel indices. The mean and standard deviation of frequency of occurrence is computed over 10 datasets generated using 10 different seeds.
   }
    \label{tab:lime_german11}   
\end{table*}

\begin{table*}[ht]
    \centering
  \begin{tabular}{p{0.5cm} p{0.9cm}p{0.9cm}p{0.9cm}p{0.9cm}p{0.9cm}p{0.9cm}p{0.9cm}p{0.9cm}p{0.9cm}p{0.9cm}p{0.9cm}p{0.9cm}p{0.9cm}}
    \toprule
    \multirow{2}{*}{Features} &
      \multicolumn{3}{c}{SHAP (\%)} &
      \multicolumn{3}{c}{Responsibility (\%)} &
      \multicolumn{3}{c}{Holler-Packel (\%)} &
      \multicolumn{3}{c}{Deegan-Packel (\%)} \\
      & {1st} & {2nd} & {3rd}  & {1st} & {2nd} & {3rd}  & {1st} & {2nd} & {3rd}  & {1st} & {2nd} & {3rd}   \\
      \midrule
Gender & 0.02 ± 0.03& 0.56 ± 0.07& 0.03 ± 0.02& 0.25 ± 0.28& 0.42 ± 0.09& 0.25 ± 0.18& 0.18 ± 0.23& 0.33 ± 0.07& 0.25 ± 0.12& 0.18 ± 0.23& 0.26 ± 0.07& 0.15 ± 0.1 \\ 
LR & 0.82 ± 0.04& 0.02 ± 0.03& 0.03 ± 0.04& 0.23 ± 0.27& 0.43 ± 0.11& 0.26 ± 0.17& 0.15 ± 0.22& 0.35 ± 0.08& 0.27 ± 0.13& 0.15 ± 0.21& 0.3 ± 0.06& 0.14 ± 0.09 \\ 
    \bottomrule
  \end{tabular}
   \caption{This table shows the results of the SHAP attack experiment on the German Credit dataset. Each row represents the mean $\pm$ standard deviation of the frequency of occurrence of either a sensitive feature (\emph{Gender}) or an uncorrelated feature (\emph{LoanRateAsPercentOfIncome}) in the top 3 positions when ranked based on their SHAP scores, Responsibility indices, Holler-Packel indices, and Deegan-Packel indices. The mean and standard deviation of frequency of occurrence is computed over 10 datasets generated using 10 different seeds.
   }
    \label{tab:shap_german11}   
\end{table*}

\begin{table*}[t]
    \centering
  \begin{tabular}{p{0.5cm} p{0.9cm}p{0.9cm}p{0.9cm}p{0.9cm}p{0.9cm}p{0.9cm}p{0.9cm}p{0.9cm}p{0.9cm}p{0.9cm}p{0.9cm}p{0.9cm}p{0.9cm}}
    \toprule
    \multirow{2}{*}{Features} &
      \multicolumn{3}{c}{LIME (\%)} &
      \multicolumn{3}{c}{Responsibility (\%)} &
      \multicolumn{3}{c}{Holler-Packel (\%)} &
      \multicolumn{3}{c}{Deegan-Packel (\%)} \\
      & {1st} & {2nd} & {3rd}  & {1st} & {2nd} & {3rd}  & {1st} & {2nd} & {3rd}  & {1st} & {2nd} & {3rd}   \\
      \midrule
 Race & 0.0 ± 0.0& 0.96 ± 0.01& 0.04 ± 0.01& 0.91 ± 0.01& 0.09 ± 0.01& 0.0 ± 0.0& 0.82 ± 0.01& 0.16 ± 0.01& 0.02 ± 0.01& 0.82 ± 0.01& 0.16 ± 0.01& 0.02 ± 0.01 \\ 
UC1 & 1.0 ± 0.0& 0.0 ± 0.0& 0.0 ± 0.0& 0.6 ± 0.02& 0.4 ± 0.02& 0.0 ± 0.0& 0.18 ± 0.01& 0.82 ± 0.01& 0.0 ± 0.0& 0.18 ± 0.01& 0.82 ± 0.01& 0.0 ± 0.0 \\ 
UC2 & 0.0 ± 0.0& 0.0 ± 0.0& 0.0 ± 0.0& 0.0 ± 0.0& 0.0 ± 0.0& 0.0 ± 0.0& 0.0 ± 0.0& 0.0 ± 0.0& 0.0 ± 0.0& 0.0 ± 0.0& 0.0 ± 0.0& 0.0 ± 0.0 \\ 
    \bottomrule
  \end{tabular}
  \vspace{0.2cm}
   \caption{This table shows the results of the LIME attack experiment on the Compas dataset. Each row represents the mean and standard deviation of the frequency of occurrence of either a sensitive feature (\emph{Race}) or an uncorrelated feature (\emph{UC1},\emph{UC2}) in the top 3 positions when ranked based on their LIME scores, Responsibility indices, Holler-Packel indices, and Deegan-Packel indices. The mean and standard deviation of frequency of occurrence is computed over 10 datasets generated using 10 different seeds.
   }
    \label{tab:lime_compas11}   
\end{table*}

\begin{table*}[t]
    \centering
  \begin{tabular}{p{0.5cm} p{0.9cm}p{0.9cm}p{0.9cm}p{0.9cm}p{0.9cm}p{0.9cm}p{0.9cm}p{0.9cm}p{0.9cm}p{0.9cm}p{0.9cm}p{0.9cm}p{0.9cm}}
    \toprule
    \multirow{2}{*}{Features} &
      \multicolumn{3}{c}{SHAP (\%)} &
      \multicolumn{3}{c}{Responsibility (\%)} &
      \multicolumn{3}{c}{Holler-Packel (\%)} &
      \multicolumn{3}{c}{Deegan-Packel (\%)} \\
      & {1st} & {2nd} & {3rd}  & {1st} & {2nd} & {3rd}  & {1st} & {2nd} & {3rd}  & {1st} & {2nd} & {3rd}   \\
      \midrule
  Race & 0.24 ± 0.06& 0.23 ± 0.11& 0.09 ± 0.05& 0.93 ± 0.01& 0.06 ± 0.01& 0.01 ± 0.0& 0.84 ± 0.02& 0.05 ± 0.02& 0.06 ± 0.01& 0.84 ± 0.02& 0.06 ± 0.02& 0.07 ± 0.01 \\ 
UC1 & 0.29 ± 0.06& 0.15 ± 0.09& 0.11 ± 0.04& 0.67 ± 0.05& 0.26 ± 0.05& 0.07 ± 0.02& 0.17 ± 0.02& 0.43 ± 0.1& 0.23 ± 0.04& 0.17 ± 0.02& 0.5 ± 0.09& 0.13 ± 0.03 \\ 
UC2 & 0.29 ± 0.06& 0.18 ± 0.08& 0.13 ± 0.06& 0.64 ± 0.05& 0.27 ± 0.05& 0.08 ± 0.04& 0.17 ± 0.02& 0.41 ± 0.1& 0.23 ± 0.04& 0.17 ± 0.02& 0.47 ± 0.1& 0.13 ± 0.04 \\ 
    \bottomrule
  \end{tabular}
    \vspace{0.2cm}
   \caption{This table shows the results of the SHAP attack experiment on the Compas dataset. Each row represents the mean $\pm$ standard deviation of the frequency of occurrence of either a sensitive feature (\emph{Race}) or an uncorrelated feature (\emph{UC1},\emph{UC2}) in the top 3 positions when ranked based on their SHAP scores, Responsibility indices, Holler-Packel indices, and Deegan-Packel indices. The mean and standard deviation of frequency of occurrence is computed over 10 datasets generated using 10 different seeds.
   }
    \label{tab:shap_compas11}  
    \vspace{-5mm}
\end{table*}

\subsection{Attack Model Details}
We now describe the implementation of the adversarial LIME and SHAP attacks \citep{Slack2020Shap}. 
Recall that the adversarial attack model for both, LIME and SHAP attacks, consists of three main components, i.e., the biased classifier, the unbiased classifier, and an out-of-distribution (OOD) classifier.
We discuss the construction of these three components below.

In all our experiments, we construct the biased classifier as a single decision layer that predicts the label based on the sensitive feature of the dataset. In the case of the Compas dataset, the biased classifier predicts the target label for each defendant based on their \emph{race}, whereas, for the German Credit dataset, the biased classifier predicts whether a candidate is good or bad solely based on the \emph{gender} of the candidate.
The unbiased classifier, on the other hand, is also a single decision layer that predicts the target label based on features that are uncorrelated with the sensitive feature in the dataset. In the case of the Compas dataset, we consider two instances, one with a single synthetic uncorrelated feature and another with two synthetic uncorrelated features. 
In experiments where we have two uncorrelated features, the unbiased classifier predicts the target label based on the XOR value of the two uncorrelated features. 
Similarly, in the case of the German Credit dataset, the unbiased classifier uses \emph{LoanRateAsPercentOfIncome} as the uncorrelated feature for predicting the target label. 

To train the OOD classifier, we need a dataset that consists of both, in-distribution and OOD data points.  In the LIME attack experiment, we construct a perturbed dataset by sampling perturbations from the standard Multivariate Normal distribution $\mathcal{N}(\vec{0},\vec{1})$ and adding it to each point in the original train dataset. On the other hand, in the SHAP attack experiment, we choose a random subset of features in each record and replace them with the values from the background distribution. Here, the background distribution is learned from the data using the K-means clustering algorithm with 10 clusters. We re-label each data point of the perturbed dataset as \emph{OOD} and each data point of the original dataset as \emph{not OOD}. We append the perturbed dataset to the original dataset to construct a new dataset. We use this new dataset to train an OOD classifier to predict if a given data point is OOD.

In both, LIME and SHAP attack experiments, we use the standard Sklearn XGBoost trees implementation \citep{scikit-learn} with $m$ estimators to train the OOD classifier, where, $m$ varies for every dataset.

As stated earlier in \cref{sec:expts}, the final LIME and SHAP attack models use their respective OOD classifiers to determine if the input data point is OOD or not. 
If the data point is OOD, the attack model uses the unbiased classifier to predict the label for the input; otherwise, it uses the biased classifier.

Our code for the adversarial attack model is an adaptation of the publicly available code for attacks on LIME and SHAP \citep{Slack2020Shap,silva2022treeensembles, inms-corr19, inms-rcra20}

\subsection{Generating Abductive Explanations for the Adversarial Attack Model}\label{sec:additional_details}

Generating the set of all abductive explanations for a given data point is intractable in theory, due to the exponential 
number of explanations in the worst-case for most of the classification models. Fortunately in practice, the number of explanations is often not large and listing the complete set of explanations can be achieved in a short/practical time.
The most effective approach to enumerate abductive explanations 
is the MARCO algorithm~\cite{LiffitonPMM16} that exploits the hitting set duality between abductive and contrastive (also referred as counterfactual)\footnote{Contrastive explanations broadly provide what changes should be made in the input data to flip the prediction.} explanations~\cite{IgnatievNA020}. 
Intuitively, the algorithm iteratively calls a SAT oracle to pick a candidate set of features for either finding one abductive or one contrastive explanation. The resulting explanation is then used to block future assignments in the SAT formula from repeating identified in the next iterations. 
%

\subsection{Additional Implementation Details}

See Tables \ref{tab:compas_hyperparameters} and \ref{tab:german_hyperparameters} for the adversarial models and datasets. 
 \begin{table}[]
        \centering
\begin{tabular}{ p{6cm}p{7cm}}
 \hline
 \multicolumn{2}{c}{Compas Dataset (1 and 2 uncorrelated feature)} \\
 \hline
 Parameters & Values\\
 \hline
 \# Train data points (LIME Attack) & 266592 \\
 \# Test data points (LIME Attack) & 618\\
  \# Train data points (SHAP Attack) & 630432 \\
 \# Test data points (SHAP Attack) & 618 \\
 OOD classifier train accuracy (LIME Attack with 1 uncorrelated feature) & 0.99 \\
 OOD classifier train accuracy (LIME Attack with 2 uncorrelated features) & 0.99 \\
 OOD classifier train accuracy (SHAP Attack with 1 uncorrelated feature) & 0.923 \\
 OOD classifier train accuracy (SHAP Attack with 2 uncorrelated features) & 0.931 \\  
  OOD classifier test accuracy (LIME Attack with 1 uncorrelated feature) & 0.849 \\
  OOD classifier test accuracy (LIME Attack with 2 uncorrelated features) & 0.843 \\
 OOD classifier test accuracy (SHAP Attack with 1 uncorrelated feature) & 0.854 \\
  OOD classifier test accuracy (SHAP Attack with 2 uncorrelated features) & 0.855 \\  
  \# Percent of OOD points (LIME Attack) & 0.5 \\
 \# Percent of OOD points (SHAP Attack) & 0.26 \\
 Features & age, two\_year\_recid, priors\_count, length\_of\_stay, c\_charge\_degree\_F, c\_charge\_degree\_M, sex\_Female, sex\_Male, race, unrelated\_column\_one, unrelated\_column\_two \\
 Features perturbed in Lime attack & age, priors\_count, length\_of\_stay \\
 Features perturbed in SHAP attack & two\_year\_recid, priors\_count, length\_of\_stay, c\_charge\_degree\_F, c\_charge\_degree\_M, sex\_Female, sex\_Male, race, unrelated\_column\_one, unrelated\_column\_two \\
 OOD classifier model for Lime Attack & Sklearn's Xgboost Classifier (n\_estimators=100, max\_depth=3,  max\_depth: 3, random\_state:10, seed:10) \\
 OOD classifier model for SHAP Attack & Sklearn's Xgboost Classifier (n\_estimators=100, max\_depth=3,  max\_depth: 3, random\_state:10, seed:10) \\
 Sensitive Feature & race \\
 Uncorrelated Features & unrelated\_column\_one, unrelated\_column\_two \\
 \hline
\end{tabular}
   \vspace{0.2cm}
        \caption{Hyperparameters used in Lime attack and SHAP attack experiments for Compas dataset}
        \label{tab:compas_hyperparameters}
    \end{table}

\begin{table}[]
    \centering
\begin{tabular}{ p{6cm}p{7cm}}
 \hline
 \multicolumn{2}{c}{German Dataset} \\
 \hline
 Parameters & Values\\
 \hline
 \# Train data points (LIME Attack) & 43200 \\
 \# Test data points (SHAP Attack) & 100\\
  \# Train data points (LIME Attack) & 47200 \\
 \# Test data points (SHAP Attack) & 100 \\
 OOD classifier train accuracy (LIME Attack) & 0.9998 \\
  OOD classifier test accuracy (LIME Attack) & 1.0 \\
 OOD classifier train accuracy (SHAP Attack) & 0.996 \\
  OOD classifier test accuracy (SHAP Attack) & 0.86 \\  
   \# Percent of OOD points (LIME Attack) &  0.5\\
 \# Percent of OOD points (SHAP Attack) & 0.85\\
 Features & ForeignWorker, Age, LoanAmount, NumberOfLiableIndividuals, Gender, CheckingAccountBalance\_geq\_200, LoanDuration, YearsAtCurrentHome, HasGuarantor, NumberOfOtherLoansAtBank, OtherLoansAtStore, LoanRateAsPercentOfIncome
\\
Features perturbed in LIME attack & Age, LoanAmount, NumberOfLiableIndividuals, LoanDuration,YearsAtCurrentHome, NumberOfOtherLoansAtBank
 \\
 Features perturbed in SHAP attack & Age, LoanAmount, NumberOfLiableIndividuals, LoanDuration,YearsAtCurrentHome, NumberOfOtherLoansAtBank\\
 OOD classifier model for Lime Attack & Sklearn's Xgboost Classifier (n\_estimators=50, max\_depth=3,  max\_depth: 3, random\_state:10, seed:10) \\
 OOD classifier model for SHAP Attack & Sklearn's Xgboost Classifier (n\_estimators=50, max\_depth=3,  max\_depth: 3, random\_state:10, seed:10) \\
  Sensitive Feature & Gender \\
 Uncorrelated Features & LoanRateAsPercentOfIncome \\
 \hline
\end{tabular}
  \vspace{0.2cm}
 \caption{Hyperparameters used in Lime attack and SHAP attack experiments for German Credit dataset}
\label{tab:german_hyperparameters}
\end{table}

\subsection{Code}
The code for reproducing the results can be found at \href{https://shorturl.at/tJT09}{https://shorturl.at/tJT09}.

\subsection{Machine Specifications}
We ran all the experiments on MacBook Air (M2 2022) with 16GB Memory and 8 cores. The total computational time $\sim$24 hours.

\end{document}